\documentclass{article}

 \usepackage[preprint]{neurips_2026}

\usepackage[utf8]{inputenc}
\usepackage[T1]{fontenc}
\usepackage{hyperref}
\usepackage{url}
\usepackage{booktabs}
\usepackage{amsfonts}
\usepackage{nicefrac}
\usepackage{microtype}
\usepackage{xcolor}
\usepackage{graphicx}
\usepackage{subcaption}
\usepackage{multirow}
\definecolor{goldenpoppy}{rgb}{0.99, 0.76, 0.0}
\definecolor{glaucous}{rgb}{0.38, 0.51, 0.71}
\definecolor{blue-violet}{rgb}{0.54, 0.17, 0.89}

\usepackage{amsmath}
\usepackage{amssymb}
\usepackage{mathtools}
\usepackage{amsthm}

\usepackage[capitalize,noabbrev]{cleveref}

\usepackage{tikz}

\usetikzlibrary{automata, positioning}

\theoremstyle{plain}
\newtheorem{theorem}{Theorem}[section]

\newtheorem{lemma}[theorem]{Lemma}

\theoremstyle{definition}
\newtheorem{definition}[theorem]{Definition}
\newtheorem{assumption}[theorem]{Assumption}
\theoremstyle{remark}

\title{
Revisiting Policy Gradients for Restricted Policy Classes: Escaping Myopic Local Optima with \\$k$-step Policy Gradients
}

\author{
  Alex DeWeese\quad Guannan Qu  \\
  Department of Electrical and Computer Engineering\\
  Carnegie Mellon University\\ Pittsburgh, PA 15213-3890, USA\\
  \texttt{\{mdeweese, gqu\}@andrew.cmu.edu}
}

\begin{document}

\maketitle

\begin{abstract}

  This work revisits standard policy gradient methods used on restricted policy classes, which are known to get stuck in suboptimal critical points. We identify an important cause for this phenomenon to be that the policy gradient is itself fundamentally myopic, i.e. it only improves the policy based on the one-step $Q$-function.
In this work, we propose a generalized $k$-step policy gradient method that couples the randomness within a $k$-step time window and can escape the myopic local optima in MDPs with restricted policy classes.  We show this new method is theoretically guaranteed to converge to a solution that is exponentially close in performance to the optimal deterministic policy with respect to $k$.
Further, we show projected gradient descent and mirror descent with this $k$-step policy gradient can achieve this exponential guarantee in $O(\frac{1}{T})$ iterations, despite only assuming smoothness and differentiability of the value function. This will provide near optimal solutions to previously elusive applications like state aggregation and partially observable cooperative multi-agent settings.
Moreover, our bounds avoid the ubiquitous distribution mismatch factors $||d_\mu^{\pi^*} / d_\mu^{\pi}||_\infty$ and $||d_\mu^{\pi^*} / \mu||_\infty$ enabling the $k$-step policy gradient method to escape suboptimal critical points that emerge from poor exploration in fully observable settings.
\end{abstract}
\section{Introduction}
\label{introduction}
Policy gradient methods are one of the most widely used methods to solve complex MDP problems
\cite{williams1992simple, mnih2016asynchronous, schulman2015trust, schulman2017proximal, lillicrap2015continuous}.
When it comes to the convergence of policy gradient methods, it has been shown that without restrictions on the  policy class, policy gradient methods will converge to the optimal policy when the initial state distribution is sufficiently exploratory \cite{bhandari2024global,agarwal2021theory, xiao2022convergence}. However, policy gradient methods on restricted policy classes in general can face poor local optima. These restricted policy classes are ubiquitous and arise in many important applications including state aggregation, cooperative multi-agent settings with independence requirements, decentralized multi-agents, group decentralized multi-agents etc. \cite{cassandra1998survey, varaiya2013max, zhao2014design,liu2017learning, qu2022scalable, deweese2024locally}.
While there have been works that show optimality guarantees of policy gradient methods for restricted policy classes, they usually work under strong assumptions, e.g. variants of Bellman completeness, closure of the restricted policy class, etc. \cite{agarwal2021theory,  bhandari2024global}. Typically, these assumptions eliminate poor suboptimal critical points and are stronger versions of ``approximate realizability'', which requires a near optimal policy (among all unrestricted policies) to be in the restricted class. However, these conditions do not hold in general for our applications of interest, which may have poor critical points and may not be approximately realizable.
\subsection{Contribution}
\textbf{Myopicness of the Policy Gradient: } In this work, we identify a source of suboptimal local optima that arise from the myopicness of the policy gradient theorem (\cref{sec:myopic}). Specifically, the policy gradient theorem finds an improvement direction based on the one-step $Q$-function, i.e. only the policy at the first step is being optimized, whereas the policy for all subsequent steps are fixed. While this can be adequate for finding the optimal policy for the unrestricted policy class (when the initial state distribution is exploratory),
we show that for restricted policy classes, no improvement direction may exist even if there are better policies in the class because of the ``short sightedness'' of the one-step $Q$-function (see \cref{running_example}).

\textbf{$k$-step Policy Gradients:} To counter this myopicness, we propose a $k$-step policy gradient method (\Cref{sec:k-step}) that can escape these poor local optima using a $k$-step $Q$-function. This method will descend on an adjacent $k$-step landscape where the critical points are guaranteed to be exponentially close to the optimal deterministic solution with respect to $k$ (see \cref{crit_guarantee}). 

\textbf{Convergence of Projected GD and Mirror Descent:} Further, we show descent (projected GD or mirror) will be guaranteed to achieve this exponentially close to optimal guarantee in $O(\frac{1}{T})$ iterations (see \cref{projected_gd_bound} and \cref{mirror_bound}). This is despite only assuming smoothness and without the usual gradient dominance condition which does not hold in general for restricted policy classes (compared to the typical $\|\nabla f(x)\|_2 \leq O(1/\sqrt{T})$ convergence guarantee for gradient descent on smooth functions).
Our results provide unusually strong theoretical convergence and near optimality guarantees rarely seen for these partial observable applications like state aggregation and cooperative multi-agent settings with independent agents, decentralized agents, or group decentralized agents (see \cref{related_works}).

\textbf{Extension to the Fully Observable Setting:}
We observe that our bounds avoid the distribution mismatch factors $||d_\mu^{\pi^*} / d_\mu^{\pi}||_\infty$ and $||d_\mu^{\pi^*} / \mu||_\infty$ so it achieves near optimal convergence even when the initial state distribution is poor. We show in \cref{all_sims} that poor exploration can cause suboptimal critical points to emerge even in fully observable settings and our $k$-step methods can escape.

\vspace{0ex}
\section{Preliminaries}
\subsection{Markov Decision Process}
The infinite horizon Markov Decision Process (MDP) includes a set of states $\mathcal S$ and actions $\mathcal A$. The dynamics begin with a sample from some fixed initial state distribution $\mu\in\Delta(\mathcal S)$. At each state, actions will be sampled from our policy $\pi: \mathcal S \rightarrow \Delta(\mathcal A)$ and will be transitioned according to a transition function $P: \mathcal S \times \mathcal A \rightarrow \Delta(\mathcal S)$ at each step.
The notation $\tau \sim \pi\lvert_\mu$ will represent the trajectory of states and actions $\tau=(s(0), a(0), s(1), a(1), \ldots)$ generated from a policy $\pi$ starting with initial state distribution $\mu$. Considering some restricted set of policies $\Pi_{res}$, we would like to find a policy $\pi \in \Pi_{res}$ that minimizes the value function $J^{\pi}(\mu) = \mathbb E_{\tau \sim \pi\lvert_\mu}[\sum_{t = 0}^\infty \gamma^t g(s(t), a(t))]$ according to some cost function $g:\mathcal S \times \mathcal A \rightarrow \mathbb R$ with $g \in [-g_{max}, g_{max}]$ and discount factor $\gamma \in (0, 1)$.
In general, solutions to the restricted policy class $\Pi_{res}$ setting are difficult to find because our standard techniques for solving MDPs may no longer apply. Performing the Bellman optimality equations may produce a solution outside of $\Pi_{res}$ and policy iteration may leave the set $\Pi_{res}$ when computing the greedy policy. Policy gradient methods appear to be a good candidate when $\Pi_{res}$ is a parameterized set of policies but the problem in general contains poor suboptimal critical points (see \cref{running_example}).
Later, we will propose a generalized $k$-step policy gradient and show  theoretical bounds that ensures various descent algorithms escape poor suboptimal local optima.

\vspace{0ex}
\section{Related Works}
\label{related_works}

To the best of our knowledge, prior works with theoretical near optimality results for policy gradients in MDPs either assume an unrestricted policy class or make strong assumptions about the policy class that heavily restrict its applications.
\textbf{Unrestricted Policy Class:} In the case of an unrestricted policy class, there have been a line of works analyzing the convergence of policy gradient methods \cite{agarwal2021theory, xiao2022convergence, yuan2022general,  fatkhullin2023stochastic, zhang2020variational, fazel2018global, liu2020improved, lan2023policy, wang2022policy, yuan2022linear, cen2022fast, mei2020global,mei2022role}. These works analyze projected gradient descent and mirror descent with its variants showing theoretical bounds for convergence. However, in our work we consider restricted policy classes which may have poor suboptimal local minima even for simple examples (see \cref{running_example}).
\textbf{Restricted Policy Class:} Compared to the unrestricted case, strong theoretical results for policy gradient methods in restricted policy classes have been minimal. This literature uses assumptions that rule out all suboptimal local minima ensuring global optimality or have bounds that are meaningful only when similar assumptions hold \cite{agarwal2021theory, bhandari2024global, mei2023ordering, wang2019neural}.
 These guarantees have ``strong results under strong assumptions'' in the sense that strong convergence rates and optimality guarantees are achieved, under strong assumptions like gradient dominance or policy-class closure. Unfortunately, these assumptions rule out many applications, including the ones considered in this work (\cref{applications}).
\textbf{Function Approximation in RL:} There is a long line of work in reinforcement learning theory that studies imposing a function class on either the $Q$-function or the policy class and seeks to find the optimal $Q$-function or policy with sample complexity polynomial in the dimension of the class as opposed to the state space size \cite{wang2019optimism, munos2005error, yang2019sample, du2021bilinear, jin2023provably, modi2020sample, agarwal2020flambe, du2019provably, jiang2017contextual, sun2019model, wen2013efficient}. However, most of these works require certain assumptions on the function class (e.g. linear MDP, linear Bellman completeness, bilinear class) all of which require the ``realizability'' or ``approximate realizability'' assumption, meaning a (near) optimal policy or $Q$-function is in the class. As mentioned in \cref{introduction}, many applications do not satisfy this, and our work does not require approximate realizability.
\textbf{Action Chunking:} Our solution has some resemblances to the recent proposal of action chunking for reinforcement learning \cite{li2025reinforcement,li2025decoupled}. However, our sampling scheme differs from action chunking in that we will be executing a sampled deterministic policy for $k$-steps rather than $k$ actions determined all at once.
\textbf{POMDP / Dec-POMDP:} Many of the applications considered in this work (such as state aggregation and decentralized multi-agents) are partially observable. Previously, the primary theoretical model for these applications have been the POMDP \cite{papadimitriou1987complexity,cassandra1998survey} and Dec-POMDP \cite{oliehoek2016concise, bernstein2002complexity} which are generally seen as intractable to solve, proven to be PSPACE-complete and NEXP-complete respectively. Our work will provide an encouraging positive result for this literature by proving strong theoretical near optimality guarantees for gradient descent methods on a restricted policy class that models the partial observability.

\vspace{0ex}
\section{Myopic Traits of Policy Gradients}\label{sec:myopic}

\subsection{Revisiting the Policy Gradient Theorem}

The $k$-step policy gradient we introduce in this work will counteract the myopic behavior of traditional policy gradient methods allowing it to better escape poor local optima.
To illustrate the issue, assume a restricted parameterized policy class $\Pi_{res}$ with policies of the form $\pi_\alpha:\mathcal S \rightarrow \Delta(A)$ for some learning parameters $\alpha \in I$. Consider the standard policy gradient theorem:
\begin{theorem}[Policy Gradient Theorem]
Let $\Pi_{res}$ be a restricted policy class parametrized by $\alpha \in I$ where $J^{\pi_\alpha}(\mu)$ is differentiable with respect to $\alpha$. Then,
\label{pg}
\vspace{0ex}

\hspace{15ex}$  \nabla_\alpha J^{\pi_\alpha}(\mu) =     \frac{1}{1 - \gamma}\mathbb E_{s \sim d_\mu^{\pi_\alpha}}[\mathbb E_{a \sim \pi_\alpha(\cdot \rvert s)}[Q^{\pi_\alpha}(s,a) \nabla_\alpha\log \pi_\alpha(a\lvert s)]]\nonumber$

where $d_{\mu}^{\pi_\alpha}(s) = (1 - \gamma)\sum_{t = 0}^\infty \gamma^t P(s_t = s\lvert \mu, \pi_\alpha)$ is the discounted state occupancy measure.
\end{theorem}
\vspace{0ex}

Notice the expression for the gradient contains the standard 1-step $Q$-function. This $Q$-function is myopic in that only one step of a chosen action is taken before reverting to the policy $\pi_\alpha$. This myopicness is more apparent when compared to the policy gradient theorem for bandits $\nabla_\alpha J^{\pi_\alpha} = \mathbb E_{a \sim \pi_\alpha(\cdot)}[g(a) \nabla_\alpha\log \pi_\alpha(a)]$.
The representation takes a familiar form with $g$ replacing $Q^{\pi}$ and minor adjustments for the lack of a state or discount factor. This similarity begs the question of whether the policy gradient theorem for MDPs is effectively using information past the first timestep to create a satisfactory descent direction.
Indeed the myopicness presents some challenges when considering restricted policy classes.
We show this using our running example below.
\vspace{0ex}
\subsection{Running Example}
\label{running_example}

\begin{figure}[]
     \centering
     \begin{subfigure}[b]{0.45\textwidth}
                 \hspace{-8ex}\includegraphics[width=1.40\linewidth]{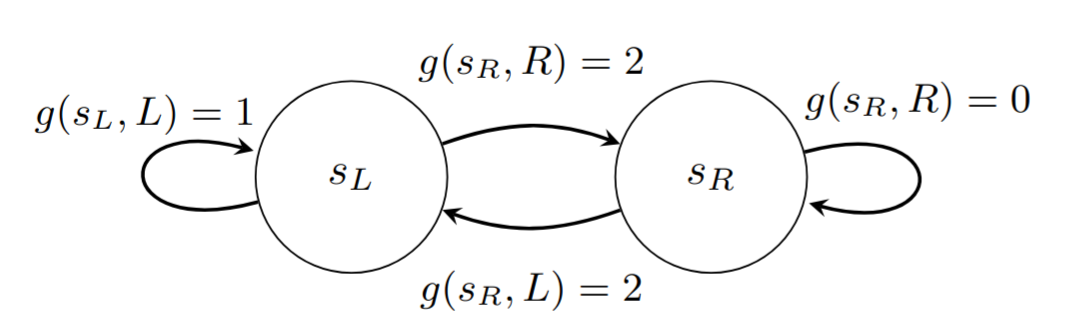}
             \vspace{3ex}
    \caption{Example from \cref{running_example}. Transitions are deterministic and incur costs shown above.}
    \label{two_state}
     \end{subfigure}
     \hfill
     \begin{subfigure}[b]{0.45\textwidth}
             \centering    \includegraphics[width=0.80\linewidth]{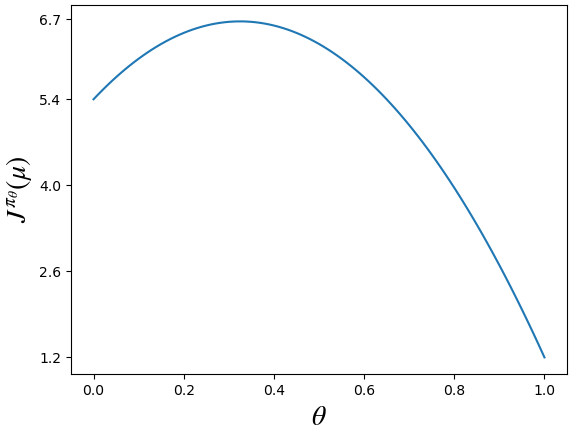}
    \caption{The value function $J^{\pi_\theta}(\mu)$ versus $\theta$ for the example in \cref{running_example}. There are two local minima at $\theta = 0, 1$ and a local maxima at $\theta \approx 0.32$.}
    \label{two_state_graph}
     \end{subfigure}
\end{figure}

To illustrate many of the points introduced in this work, we will refer frequently to a two state MDP example borrowed from \cite{bhandari2024global} seen in \cref{two_state}. Despite its minimal construction, it will contain a suboptimal critical point as seen in \cref{two_state_graph}.
In \cref{two_state}, the system can be in state $s_L$ or $s_R$ and transitioned between each other using a ``left action'' $L$ or ``right action'' $R$. Changing the state will incur cost 2 and staying at $s_L$ will incur a cost of 1.
The policy class is restricted to $\Pi_{res} = \{\pi_\theta: \theta\in[0,1]\}$ where $\pi_\theta$ takes action $R$ with probability $\theta$ regardless of the state and action $L$ otherwise.
With $\gamma = 0.8$ and $\mu = [0.6, 0.4]$ this creates a suboptimal local minima at $\theta = 0$ as seen in \cref{two_state_graph}.
To show why this is the case, when $\theta \approx 0$, the policy gradient theorem simplifies to: $\nabla_\theta J^{\pi_\theta}(\mu)$
$ \approx 4.6 (Q^{\pi_0}(s_L, R) - J^{\pi_0}(s_L))+ 0.4 (Q^{\pi_0}(s_R, R) - J^{\pi_0}(s_R)).$
Following \cref{two_state} to evaluate $Q^{\pi_0}(s_L, R)$, the trajectory starts at $s_L$ and initially takes action R, moving the state to $s_R$ and incurring a large cost. This is followed by the action L which will move back to $s_L$ incurring another large cost then remaining at $s_L$. This performs worse than $J^{\pi_0}(s_L)$ which simply remains at state $s_L$.
Therefore $Q^{\pi_0}(s_L, R) - J^{\pi_0}(s_L) > 0$ and the dominant first term turns positive indicating a positive slope and a suboptimal critical point at $\theta = 0$.
We later introduce a new gradient update utilizing a $k$-step $Q$-function that will allow for $k$-steps of evaluation of action R prior to returning to $s_L$. When action R is taken $k$ times, it gives the opportunity to reduce the cost by remaining at the no-cost state $s_R$ for longer periods of time before returning to $s_L$.
This allows the $Q$-function to ``adequately test'' a policy as part of the gradient computation and achieve near optimal performance for various descent algorithms.
\vspace{0ex}
\subsection{Origin of Myopicness}
\label{origin}
At first, the myopicness of the gradient may seem mysterious because the gradient straightforwardly presents the direction of steepest ascent and seems like it could always be used to create a good update direction. However, it is specifically how the gradient interacts with the independent sampling of the policy in the traditional 1-step evaluation scheme that causes myopicness. Intuitively, trajectories that use ``multi-step reasoning'' will be lost in higher order terms.
Concretely, in the two state MDP example from \cref{two_state}, suppose we begin at $\pi_0$ (i.e. $\theta=0$, policy always takes action L) and increment the parameter by a small $d\theta$. The policy $\pi_{d\theta}$ then takes action R with this infinitesimal $d\theta$. Further, we restrict the horizon to be a large H such that $\frac{\gamma^H}{1 - \gamma} g_{max}$ is small.
Notice that for $\mathbb E_{\tau \sim \pi_{d\theta}\lvert_\mu}[\sum_{t = 0}^{H - 1} \gamma^t g(s(t), a(t))]$
any trajectory that takes action R $m$-times will have a probability of $(d\theta)^m (1 - d\theta)^{H - m} \leq (d\theta)^m$ due to independence. Therefore, trajectories with $m \geq 2$ are lost in the higher order terms and will not appear as part of the (right hand side) derivative at 0 because $\lim_{\Delta\theta \rightarrow 0}\frac{(\Delta \theta)^m}{\Delta \theta} = 0$.
In other words, only trajectories which take action $R$ at most one time are considered as part of the computation. This then provides an intuition for the reason the 1-step $Q$-function appears in the gradient expression.
The primary issue that causes this myopicness is this independent evaluation of the policy at each step, that does not allow the gradient to ``see the effect'' of the change $d\theta$ across multiple steps.

\vspace{0ex}
\section{Breaking Myopicness with $k$-step PG}\label{sec:k-step}
\subsection{Correlated Policies}
\label{policy_def}

In this work, we introduce a new interpretation for the policy. In the standard definition of the MDP, stochastic policies are of the form $\pi:\mathcal S \rightarrow \Delta(\mathcal A)$ which takes an independent action at each state. However, another way of viewing a stochastic policy is as a distribution of deterministic policies.

\begin{definition}[Correlated Policy]
    Let $\Pi_{det}$ be a set of deterministic policies of the form $\pi_{det}: \mathcal S \rightarrow \mathcal A$. A correlated policy $\tilde\pi$ is a distribution over $\Pi_{det}$ or in other words, $\tilde\pi\in \Delta(\Pi_{det})$. An action can be taken at state $s$ by sampling $\pi_{det} \sim \tilde\pi$ and taking $\pi_{det}(s)$.
\end{definition}

Going forward, the restricted set of these correlated policies will be denoted $\tilde \Pi_{res}$.
The restricted policy class from the two state example in \cref{running_example}, can be viewed from a correlated policy perspective. Recall $\pi_L$ is the deterministic policy that takes action L at both states and $\pi_R$ the corresponding policy for action R.  The restricted policy class can then be described as $\Pi_{det} = \{\pi_L, \pi_R\}$ and $\tilde\Pi_{res} = \{\tilde\pi_\theta\lvert \tilde\pi_\theta = I_\theta \pi_R + (1 - I_\theta)\pi_L, \theta \in [0,1]\}$ where $I_\theta$ is an indicator function that is 1 with probability $\theta$ and 0 otherwise.
In a standard policy rollout, our proposed formulation would \emph{independently} at each step, sample $\pi_{det}\sim \tilde\pi$ and take $\pi_{det}(s)$ for current state $s$. This is equivalent to taking the traditional policy $\pi:\mathcal S \rightarrow \Delta(\mathcal A)$ defined implicitly by $\pi_{det}(s)$ where $\pi_{det}\overset{\mathrm{iid}}{\sim} \tilde\pi$. Therefore in the ``1-step evaluation'' in standard rollout (independent evaluations at each step), the correlated policy and its corresponding traditional policy are equivalent.
Later, we will introduce a ``$k$-step evaluation'' where this new representation will be required.

\vspace{0ex}
\subsection{Restriction Assumption}
\label{assumption}

For this work, we consider an arbitrarily restricted set of deterministic policies $\Pi_{det}$ determined by our application. The main assumption made is that the correlated policy class $\tilde\Pi_{res}$ is all possible distributions $\Delta(\Pi_{det})$. This ensures sufficient expressibility during the ``training phase.''
\begin{assumption}
    $\Pi_{det}$ is an arbitrary set of deterministic policies and $\tilde\Pi_{res}$ is all policies $\Delta(\Pi_{det})$.
\end{assumption}
The general outline of our solution framework is to obtain a restricted set of deterministic policies $\Pi_{det}$ from our problem description which are of the form $\pi_{det}: \mathcal S \rightarrow \mathcal A$. Then a $k$-step optimization procedure is performed over the correlated policies $\tilde \Pi_{res}$ to obtain a solution $\tilde\pi \in \tilde \Pi_{res}$. This may have undesired correlations (e.g. among states) so the sample $\pi_{det} \sim \tilde \pi$ will be presented as a solution which again has the form of a traditional uncorrelated policy $\pi_{det}: \mathcal S \rightarrow \mathcal A$.
The following is a minimal list of applications that our work will encompass.

\vspace{0ex}
\subsection{Applications}
\label{applications}
\textbf{Policy with State Aggregation:}
Consider an MDP environment with a set of observations $\mathcal O$ and an observation function $\mathcal O_{func}:\mathcal S \rightarrow \mathcal O$. We may define the set of deterministic policies that only depend on this observation. $\Pi_{det}$ will therefore be all possible deterministic policies of the form $\pi_{det}(s) = \pi_{det}^{+}(\mathcal O_{func}(s))$ where $\pi_{det}^{+}:\mathcal O \rightarrow \mathcal A$.
The two state example from \cref{running_example} is a state aggregation problem.
Real life scenarios involving policies with state-aggregation may appear by a necessity to reduce computation need (e.g. by grouping similar states into the same aggregated state). Other times, state aggregation may be a result of partial observability. For example, the state is a tuple $(x,y)$, but the agent only observes $x$.
These types of partially observable systems appear in applications in the medical field, computer vision, military applications and many more (see \cite{cassandra1998survey}). This is typically modeled as a POMDP which is generally intractable to solve. With the methods developed in this work, we will be able to produce unusually strong theoretical near optimality results for this setting.

\textbf{Independent Multi-Agents:} Frequently in multi-agent settings, every agent has its own state $\mathcal S_i$ which produces a global state $\mathcal S = \mathcal S_1 \times \ldots \times \mathcal S_n$ for agents $\mathcal N = \{1, \ldots, n\}$. Each agent then decides on an action $a_ i \in \mathcal A_i$ with a global action space of $\mathcal A = \mathcal A_1 \times \ldots \times \mathcal A_n$. Further, in cooperative independent agent environments, agents are only allowed to observe their own state but have rewards and transitions that may depend on other agents. Examples include traffic control \cite{varaiya2013max} and power systems \cite{zhao2014design}. Unfortunately, the general multi-agent extension of the POMDP called the Dec-POMDP is NEXP-Complete and generally intractable to solve \cite{bernstein2002complexity}.
In these environments, $\Pi_{det}$ will be all deterministic policies that map to $\pi_{det}(s) = (\pi_{det}^i(s_i): i \in \mathcal N)$ for all possible individual policies $\pi_{det}^i: \mathcal S_i \rightarrow \mathcal A_i$ for $i \in\mathcal N$ (see \cref{num_match} for an example).

\textbf{Decentralized Multi-Agents:}
Another common multi-agent setting is when cooperative decentralized agents act according to an observation of the state. Assume the same factored state and action space as in the independent multi-agents setting. Then at a state $s\in \mathcal S = \mathcal S_1\times \ldots \mathcal S_n$, agent $i$ acts according to some observation $O_{i,func}(s)$ where $\mathcal O_{i, func}: \mathcal S \rightarrow \mathcal O_i$ for some set of observations $\mathcal O_i$.
The restricted set of deterministic policies $\Pi_{det}$ is then $\pi_{det}(s) = (\pi^{i,+}_{det}(\mathcal O_{i,func}(s)): i\in \mathcal N)$ for all possible individual policies $\pi_{det}^{i,+}: \mathcal O_i \rightarrow \mathcal A_i$ for $i \in \mathcal N$ (see \cref{button} for an example).
Similarly, these applications in general  are also modeled using a Dec-POMDP which is generally seen as intractable to solve. Examples include robotics applications where agents can only view other agents within their vicinity \cite{omidshafiei2015decentralized, liu2017learning}. This also encompasses the networked RL literature where agents are connected in some graph with a fixed distribution and agents can observe their $\kappa$-hop neighbors \cite{qu2022scalable, qu2020scalableaverage, lin2021multi}.

\textbf{Group Decentralized Agents:} In the group decentralized setting, agents are grouped together based on some criteria, and the agents in that group all share a joint observation of each others state. For $s \in \mathcal S = \mathcal S_1 \times \ldots \times \mathcal S_n$, a grouping function exists $G:\mathcal S \rightarrow Part(\mathcal N)$ where $Part(\mathcal N)$ is the set of partitions over the agents. Then at each state $s$, each agent in the same partiton $g \in G(s)$ will observe the states of the agents in the group $s_g = (s_i: i \in g)$. $\Pi_{det}$ will be all deterministic policies of the form $\pi_{det}(s) = (\pi^g_{det}(s_g): g \in G( s))$ with $\pi_{det}^g:\mathcal S_g \rightarrow \mathcal A_g$ and $g \subset \mathcal N$.
The concept of group decentralized agents is expanded upon in the Locally Interdependent Multi-Agent MDP works \cite{deweese2024locally,deweese2025thinking} that can model various autonomous vehicle and robotics applications.
The outcome of policy gradient methods directly on the group decentralized policy class is still not well understood and the methods in this work can offer a solution.

\subsection{The $k$-step Evaluation}
Leveraging our new formulation of the policy, we will now introduce the $k$-step evaluation of a policy.
In the traditional 1-step evaluation described in more detail in \cref{policy_def} we sample $\pi_{det}\sim \tilde\pi$ from $\tilde\pi\in\tilde\Pi_{res}$ and take the policy $\pi_{det}(s)$ independently at each step. For the $k$-step evaluation, once $\pi_{det}\sim \tilde\pi$ is sampled, we will execute $\pi_{det}$ for $k$ iterations before sampling again. This will allow for a better ``measure of the quality'' of the deterministic policies sampled by $\tilde\pi$.
\begin{definition}[$k$-step value function]
For a correlated policy $\tilde\pi\in \tilde\Pi_{res}$, the $k$-step value function is $J^{\tilde\pi,k}(s) = \mathbb E_{\tau \sim \tilde\pi\lvert_{s,k}}[\sum_{t = 0}^\infty \gamma^t g(s(t),a(t))]$ where $\tau \sim \tilde\pi\lvert_{s, k}$ represents the trajectory starting at state $s$ generated by repeatedly sampling $\pi_{det}\sim\tilde\pi$ and taking $\pi_{det}$ for $k$ steps (we sample from $\tilde \pi$ at timesteps $0,k ,2k,\ldots$ and the sampled deterministic policy is used in the intermediate timesteps).
\end{definition}
Notice that  this $k$-step value function is still a function on $\mathcal S$ but it uses correlated policies and is rolled out in a different way, altering the probabilities of the sampled trajectories. We now define the analogous $k$-step $Q$-function.

\begin{definition}[$k$-step $Q$-function]

For a correlated policy $\tilde\pi \in \tilde\Pi_{res}$, the $k$-step $Q$-function is  $Q^{\tilde\pi,k}(s, \pi_{det}') = \mathbb E_{\tau \sim \tilde\pi\lvert_{s, \pi_{det}', k}}[\sum_{t = 0}^\infty \gamma^t g(s(t),a(t))]$ where $\tau \sim \tilde\pi\lvert_{s, \pi'_{det}, k}$ represents the trajectory starting at state $s$ and takes $\pi'_{det}$ for the initial $k$ steps before repeatedly using $\tilde\pi$ to sample and execute for $k$ iterations ($\pi_{det}'$ is used for the first $k$ timesteps followed by taking samples of $\tilde\pi$ at timesteps $k, 2k,\ldots$ and using the sampled deterministic policy for the intermediate timesteps).

\end{definition}
\label{tradeoff}
\textbf{Computing the $k$-step $Q$-function.}
This $Q$-function should not be computed in whole as the second argument is dependent on a potentially large number of deterministic policies in $\Pi_{det}$ (much more than the number of actions as in the traditional $Q$-function).
In practice, we suggest building an estimator for $J^{\tilde\pi}(s)$ and choosing a reasonable $k$ for the problem so $Q^{\tilde\pi,k}(s, \pi'_{det})$ can quickly be reconstructed. Notice in the case of deterministic dynamics $Q^{\tilde\pi,k}(s, \pi'_{det})$ can be constructed from $J^{\tilde\pi}(s)$ quite quickly as $Q^{\tilde\pi,k}(s, \pi'_{det}) = g(s, \pi'_{det}(s)) + \gamma g(s_1, \pi'_{det}(s_1)) + \ldots + \gamma^k J^{\tilde\pi}(s_k)$ would only require a single $k$-step roll out. In non-deterministic cases, Monte-Carlo estimation can be used along with various variance reduction methods.
This $k$-step $Q$-function has a tradeoff between the computation time and myopicness. When $k = 1$, the standard $Q$-function $Q^{\tilde\pi}(s,\pi_{det}'(s))$ can be used but is mostly an evaluation of the performance of $\tilde\pi$. When $k \rightarrow \infty$, the $k$-step $Q$-function becomes $Q^{\tilde\pi,k}(s, \pi_{det}') \rightarrow J^{\pi'_{det}}(s)$ which is excellent in determining if $\pi'_{det}$ is a ``good policy'' but requires a full computation of $J^{\pi'_{det}}(s)$ for every $\pi'_{det}$ and state $s$ which is not feasible.

\vspace{0ex}
\subsection{The $k$-step Policy Gradient}
\label{k_pg}

The gradient in this $k$-step regime can be expressed with the following variant of the Policy Gradient Theorem.
\begin{theorem}[$k$-step Policy Gradient Theorem]
Let $\tilde \Pi_{res}$ be a restricted policy class parametrized by $\alpha \in I$ and $J^{\tilde \pi_\alpha,k}(\mu)$ is differentiable with respect to $\alpha$. Then the following equality holds:
\vspace{2ex}
\\\hspace*{10ex}
$\nabla_\alpha J^{\tilde\pi_\alpha,k}(\mu)=
     \frac{1}{1 - \gamma^{k}} \mathbb E_{s \sim d_{\mu}^{\tilde\pi_\alpha, k}}\mathbb E_{\pi_{det}\sim \tilde\pi_\alpha}[ Q^{\tilde\pi_\alpha, k}(s, \pi_{det})\nabla_\alpha\log\tilde\pi_\alpha(\pi_{det})].$
\label{kpg_representation}
\end{theorem}
Here, $d_{s_0}^{\tilde\pi_\alpha, k} = (1 - \gamma^k)\sum_{m = 0}^\infty \gamma^{mk} P(s_{mk} = s\lvert s_0, \tilde\pi_\alpha, k)$ is the $k$-step discounted state occupancy measure. The proof is shown in \cref{kpg_proof}.

Compared to the standard policy gradient theorem in \cref{pg}, this $k$-step policy gradient theorem incorporates descent over the correlated policies as well as including a $k$-step $Q$-function in its expression, which tests $\pi_{det}$ for $k$ steps to compute the gradient. This will overcome the myopic traits of standard policy gradient (see \cref{main_results}).
Consider using this gradient in the two state example from \cref{running_example} expressed with correlated policies in \cref{policy_def}. For $\tilde\pi_\theta$ with $\theta \approx 0$, the $k$-step $Q$-function in the gradient expression will include the value of taking action $R$ (or the deterministic policy $\pi_R$) for $k$ times before reverting to $\pi_\theta \approx \pi_L$. This will allow the $k$-step policy gradient method to overcome the myopicness of traditional policy gradients (see \cref{simulation}).

\vspace{0ex}
\section{Theoretical Guarantees}\label{theoretical_guarantees}
\label{main_results}
We show that when policies are improved with the $k$-step policy gradient introduced in \cref{k_pg}, strong theoretical bounds can be proved.
Recall that we are assuming an arbitrarily restricted class of deterministic policies $\Pi_{det}$ and learning a policy in $\tilde\Pi_{res}$ which is all correlated policies of the form $\tilde\pi\in \Delta(\Pi_{det})$ (see \cref{assumption}).
The optimality guarantees in this section are relative to the optimal deterministic policy in the policy class $\pi_{det}^* \in \Pi_{det}$. For deterministic policies, the $k$-step evaluation and the traditional 1-step evaluations are equivalent ($J^{\pi^*_{det},k}(\mu) = J^{\pi^*_{det}}(\mu)$) and therefore, this optimal deterministic policy can act as a fixed point of comparison for both settings and for all $k$.

\vspace{0ex}
\subsection{Near Optimal Critical Points} To begin, we show that all zero gradient critical points and local minima on the boundary are theoretically guaranteed to be exponentially close in performance to the optimal deterministic policy with respect to $k$. This demonstrates a strong and general near optimality bound that is rarely seen for these restricted policy classes.
\begin{theorem}
\label{crit_guarantee}
Consider the policy class $\tilde\Pi_{res}$ with the assumptions from \cref{assumption} with the trivial parametrization (elements of $\tilde \Pi_{res}$ are directly represented as elements of the simplex $\Delta(\Pi_{det})$). Assume $J^{\tilde\pi,k}(\mu)$ is differentiable with respect to $\tilde \pi$.
Let $\tilde\pi_{crit} \in \tilde\Pi_{res}$ be a zero-gradient critical point or a local minima on the boundary.
Then, the following performance bound is satisfied:
\vspace{1ex}
\\
\hspace*{25ex}$\mathbb E_{\pi_{det}\sim\tilde\pi_{crit}}[J^{\pi_{det}}(\mu)] - J^{\pi_{det}^*}(\mu) \leq 8\frac{\gamma^k}{1 - \gamma}g_{max}$
\end{theorem}

Notice that as $k$ increases, the performance of ALL specified critical points in the system are pushed towards the optimal deterministic policy exponentially fast in $k$. Further, as will be shown in the next subsections, descent using $k$-step evaluations will converge to a theoretically near optimal solution despite considering a broad class of restricted policies that may have poor 1-step local minima as in \cref{two_state}.
Lastly, we note \cref{closed_form} is used to prove this result and acts as an approximate gradient dominance condition that will take the place of the usual gradient dominance condition as in \cite{agarwal2021theory} for the following descent proofs.

\vspace{0ex}
\subsection{Projected Gradient Descent} Next, despite only assuming smoothness, performing projected gradient descent and mirror descent will satisfy this exponential performance bound (\cref{crit_guarantee}) in $O(\frac{1}{T})$ iterations. Specifically, projected gradient descent on the trivially parametrized policy $\tilde \Pi_{res}$ (in other words $\Delta(\Pi_{det})$) with updates $\tilde\pi_{t + 1} \leftarrow Proj_{\tilde \Pi_{res}}(\tilde \pi_t - \eta \nabla J^{\tilde \pi_t,k}(\mu))$ achieves the following (proved in \cref{part2_proj}).
\begin{theorem}
\label{projected_gd_bound}
Suppose $J^{\tilde\pi,k}(\mu)$ is differentiable and $\beta$-smooth in $\tilde\pi$. Then, projected gradient descent on the trivially parametrized $\tilde\Pi_{res}$ with learning rate $\eta = \frac{1}{\beta}$ converges. Further, if $\tilde\pi_{T}$ is the T-th timestep in projected gradient descent with the trivial parametrization, then the following performance guarantee is satisfied:
\\
\hspace*{15ex}$
\mathbb E_{\pi_{det}\sim\tilde\pi_{T}}[J^{\pi_{det}}(\mu)] - J^{\pi_{det}^*}(\mu) \leq  8\frac{\gamma^k}{1 - \gamma}g_{max} + \frac{1}{T}\bigg(\frac{\beta}{2}\|\tilde\pi^* - \tilde \pi_0\|_2^2\bigg)
$
\end{theorem}\vspace{0ex}
Prior works like \cite{agarwal2021theory, bhandari2024global} have achieved this rate for unrestricted policy classes where a gradient dominance condition is satisfied, whereas we may have suboptimal critical points as in \cref{running_example}.
In general, smoothness alone for projected gradient descent and mirror descent on general functions only provides convergence bounds of the form $\|\nabla J^\pi(\mu)\|_2 \leq O(\frac{1}{\sqrt T})$. In fact, in \cite{carmon2020lower, carmon2021lower} an information theoretic lower bound is shown guaranteeing this rate is tight for any algorithm on the class of all problems with bounded initial value and first-order smoothness. Our $O(\frac{1}{T})$ rate holds because the $k$-step evaluation induces an approximate gradient dominance condition (\cref{closed_form}). Therefore, this $k$-step policy gradient brings about a ``good landscape'' for the descent algorithms ensuring a faster descent to ``high performing'' solutions.

Crucially, our results are not dependent on the distribution mismatch coefficients of the form $||d_\mu^{\pi^*} / d_\mu^{\pi}||_\infty$ nor $||d_\mu^{\pi^*} / \mu||_\infty$ that is common in the literature \cite{agarwal2021theory,xiao2022convergence, yuan2022general}. This is a term that can easily be infinite when the supports of $\mu$ or $d_\mu^{\pi}$ does not subsume the supports of $d_\mu^{\pi^*}$ that depends on the unknown policy $\pi^*$. See \cref{all_sims} for examples of how suboptimal critical points can emerge in fully observable examples with poor exploration and how the $k$-step methods can escape.

\subsection{Mirror Descent}
Non-trivial parameterizations on $\tilde \Pi_{res}$ can be handled through the generalized mirror descent algorithm.
For $\tilde\pi_\alpha \in \tilde \Pi_{res}$ which is parametrized by $\alpha \in I$, mirror descent with mirror map $\Phi$ updates the policy $\tilde\pi_{\alpha_{t + 1}} \leftarrow \text{argmin}_{\tilde\pi_\alpha\in \tilde \Pi_{res}} \langle \nabla J^{\tilde\pi_{\alpha_t},k}(\mu), \tilde\pi_\alpha\rangle + \frac{1}{\eta}D_{\Phi}(\tilde\pi_\alpha, \tilde\pi_{\alpha_t})$ where $D_\Phi$ is the Bregman divergence.
We present the following result for mirror descent with these correlated policies (proved in \cref{part2_mirror}):

\begin{theorem}
\label{mirror_bound} Assume $J^{\tilde \pi,k}(\mu)$ is differentiable and $\beta$-smooth in $\tilde\pi$ with respect to some norm $\|\cdot\|$. Let $\Phi$ be a differentiable mirror map that is $\lambda$-strongly convex with respect to the same norm. Then mirror descent on the constraint set $\tilde\Pi_{res}$ with learning rate $\eta = \frac{\lambda}{\beta}$ converges. Further if $\tilde\pi_{T}$ is the T-th timestep in mirror descent. Then the following performance guarantee is satisfied:
\\
\hspace*{15ex}$
\mathbb E_{\pi_{det}\sim\tilde\pi_{T}}[J^{\pi_{det}}(\mu)] - J^{\pi_{det}^*}(\mu)
\leq 8\frac{\gamma^k}{1 - \gamma}g_{max} + \frac{1}{T}\bigg(\frac{\beta}{\lambda}D_{\Phi}(\tilde\pi^*,\tilde\pi_0)\bigg)
$
\end{theorem}\vspace{0ex}
For various parameterizations and choices of a mirror maps, the mirror descent updates often simplify.
In the standard MDP setting with an unrestricted policy class, when the expected negative Shannon entropy mirror map is used, the mirror descent update becomes the ubiquitous natural policy gradient.  The natural policy gradient handles the learning parameters through the fisher information matrix. When a softmax parametrization is used with the natural policy gradient, this results in a multiplicative-weights style update \cite{agarwal2021theory}. The natural policy gradient is also the basis for empirical algorithms such as trust region policy optimization (TRPO) and proximal policy optimization (PPO) \cite{schulman2015trust, schulman2017proximal}.
 However, for both the traditional setting and our correlated setting, the natural policy gradient update may require modifications to adhere to the policy class restriction and may have a different form dependent on the restriction $\tilde \Pi_{res}$ \cite{fox2022independent, chen2024decentralized, thomas2013projected}. Thus, for our purposes, we will focus on mirror descent in its most general form.
\vspace{0ex}
\section{Simulation of Two-State Example}
\label{simulation}
\begin{figure}
    \centering
 \hspace*{-5ex}\includegraphics[width=1.10\linewidth]{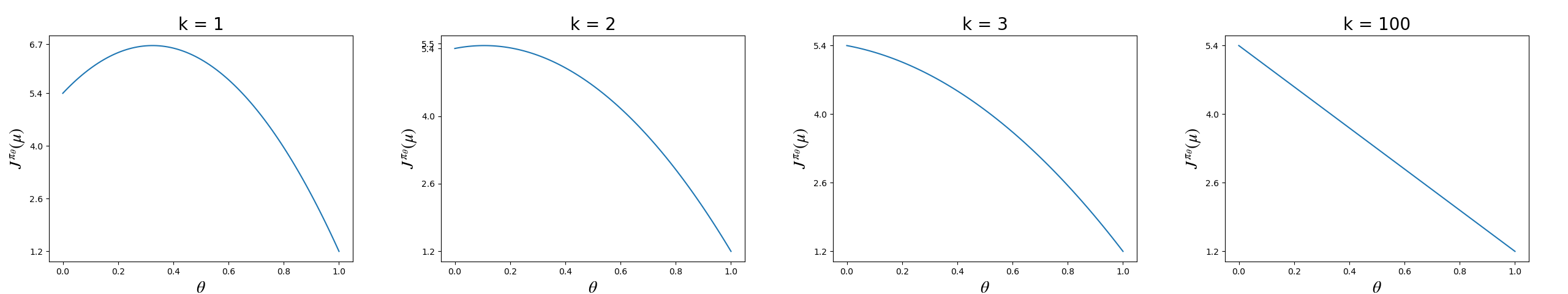}
    \caption{Critical points being removed as predicted by our theoretical bound discussed in \cref{theoretical_guarantees}.}
    \label{two_state_k}
    \vspace{0ex}
\end{figure}

We confirm our results in the two state example from \cref{running_example} shown in \cref{two_state_k} (as well as additional examples in \cref{all_sims}). In the $k = 1$ diagram, we see the 1-step evaluation from \cref{two_state} in \cref{running_example}. The three critical points are a local minima at $\theta = 0$, $\theta = 1$ and a local maximum at $\theta \approx 0.32$. As $k$ is increased, the curve straightens and the local minima at $\theta = 0$ ceases to be a local minima by $k_{esc} = 3$. Likewise the local maxima at $\theta \approx 0.3$ by $k_{esc} = 3$ becomes sloped and is no longer a local maxima. Increasing $k$ successfully removes the suboptimal critical points as predicted by our theoretical bound.
When $k$ is large as seen in the $k = 100$ diagram, the $k$-step evaluation turns affine. This is because when $k$ is large, the initial deterministic policy $\pi_{det}(s)\sim \tilde\pi$ from $\tilde\pi \in \tilde\Pi_{res}$ will determine the trajectory for longer periods. In other words, as $k \rightarrow \infty$, the value function simplifies to $J^{\tilde\pi_\theta,k}(\mu) \rightarrow \theta J^{\pi_R}(\mu) + (1 - \theta) J^{\pi_L}(\mu)$. In general, when $k$ is large, an affine structure will emerge (because of our assumptions in \cref{assumption}) creating a very easy space to optimize but the gradient will become progressively more expensive to compute as described in \cref{tradeoff}.

\vspace{0ex}
\section{Proof Sketch}
\label{proof_sketch}

The proofs in this work primarily focus on directional derivatives in MDPs. When a directional derivative is negative at a point, this guarantees that the gradient is non-zero and provides a direction for the descent algorithms to improve. Intuitively, the directional derivative provides an ``escape direction'' for the descent process.
The proofs are split into two parts. Firstly, we provide an approximate gradient dominance result. \cref{crit_guarantee} proved in \cref{part1_proof} will follow directly.
Secondly, we will show that descent (projected GD or mirror) on smooth functions will provide an average iterate bound on the negative directional derivative towards the optimal deterministic policy with a $O(\frac{1}{T})$ rate. This is again in contrast to the standard convergence rate on smooth functions that guarantees $\|\nabla J^\pi(\mu)\|_2 \leq O(\frac{1}{\sqrt T})$. Combining part 1) and part 2) will give \cref{projected_gd_bound} and \cref{mirror_bound} proved in \cref{part2_proj} and \cref{part2_mirror} respectively.

\textbf{Part 1)}
Let $\tilde\pi^*$ be the correlated policy with the Dirac delta of the optimal deterministic policy $\delta(\pi_{det}^*)\in \tilde\Pi_{res}$.
The directional derivative at the point $\tilde\pi$ towards $\tilde\pi^*$ is denoted as $\nabla_{\tilde\pi^* - \tilde\pi} J^{\tilde\pi,k}(\mu) = (\tilde\pi^* - \tilde\pi)\cdot \nabla_{\tilde\pi} J^{\tilde\pi,k}(\mu)$.
We show
in \cref{closed_form}:
\vspace{-0ex}
$
\nabla_{\tilde\pi^* - \tilde\pi} J^{\tilde\pi,k}(\mu) \leq  \frac{1}{1 - \gamma^k}   (J^{\tilde\pi^*,k}(\mu) - J^{\tilde\pi,k}(\mu))
+ \frac{6\gamma^k}{(1 - \gamma^k)(1 - \gamma)} g_{max}\nonumber
$.
This acts as an approximate gradient dominance condition and ensures that if the directional derivative $\nabla_{\tilde\pi^* - \tilde\pi_{crit}} J^{\tilde\pi_{crit},k}(\mu) \geq 0$ at a local minimum or $0$ at gradient critical point, then
$J^{\tilde\pi_{crit},k}(\mu) - J^{\pi_{det}^*}(\mu) \leq 6\frac{\gamma^k}{1 - \gamma}g_{max}$  and therefore $\mathbb E_{\pi_{det}\sim\tilde\pi_{crit}}[J^{\pi_{det}}(\mu)] - J^{\pi_{det}^*}(\mu) \leq 8\frac{\gamma^k}{1 - \gamma}g_{max}$ as desired (see \cref{part1_proof}).

\textbf{Part 2)} For general smooth functions, when performing projected gradient descent on a convex set or mirror descent with a $\lambda$-strongly convex mirror map, the negative directional derivative can be bounded by a $O(\frac{1}{T})$ term.
Intuitively, this is done through a modification of the traditional smooth + convex $O(\frac{1}{T})$ performance bound. The traditional smooth + convex proof uses the first order condition for convexity which lower bounds the negative directional derivative
$f(x) - f(y) \leq \nabla f(x)\cdot (x - y) = -\nabla_{y - x} f(x)$.
Since we will use \cref{closed_form} in the place of convexity, we will instead keep the negative directional derivative term and simplify. In the case of mirror descent with a learning rate $1/\beta$, this yields the following average iterate bound:
$
\frac{1}{T}\sum_{t = 0}^{T - 1}-\nabla_{\tilde\pi^* - \tilde \pi_t} J^{\tilde\pi_{t},k}(\mu)  \leq \frac{1}{T}\bigg(\frac{\beta}{\lambda}D_{\Phi}(\tilde\pi^*,\tilde\pi_0)\bigg) - \frac{1}{T}\sum_{t = 0}^{T - 1}(J^{\tilde\pi_{t + 1},k}(\mu) - J^{\tilde\pi_t,k}(\mu)) \nonumber
$
proved in \cref{proj_gd_lemma}.
The result from Part 1) then provides a lower bound on $-\nabla_{\tilde\pi^* - \tilde\pi_t} J^{\tilde\pi_{t}}(\mu)$. Noting that mirror descent on a smooth function reduces the value at each step (proved in \cref{descent_lemma_mirror}), yields the last-iterate bound in \cref{mirror_bound}.
\vspace{0ex}
\section{Future Works / Discussion}
\label{discussion}

Our results (\cref{main_results}) show a promising new angle of attack for solving MDPs with restricted policy classes in applications such as state aggregation, independent multi-agents, decentralized multi-agents, and group decentralized agents (\cref{applications}). This work presents a theoretical treatment, but the future empirical works seem just as promising.
Diffusion models have been used to represent the policy in various RL settings \cite{wang2022diffusion, ma2025efficient, ma2025reinforcement} and may model the potentially high dimensional correlated policies. Learning could be performed using reweighted score matching \cite{ma2025efficient} using the $k$-step $Q$-function. This motivates applications of this theory to larger scale simulations and baselines.
We may also extend the MDP results in this work to the RL setting where transitions and reward functions are unknown and further investigate the consequences of the $k$-step methods in the fully observable settings (see \cref{all_sims}) where optimization is still challenging such as in robotics applications \cite{schaul2019ray}.

  \section*{Acknowledgments}
 This research was supported by NSF Grants 2339112, 2512805, CMU CyLab Seed Funding, and Pennsylvania Infrastructure Technology Alliance. In addition, Alex DeWeese is supported by Leo Finzi Memorial Fellowship in Electrical \& Computer Engineering, the David H. Barakat and LaVerne Owen-Barakat CIT Dean's Fellowship, and the Fritsch Family Fellowship.

\bibliographystyle{plain}
\bibliography{refs}

\newpage
\appendix

\section{Simulations}
\label{all_sims}

\newcommand{\psub}{\pi_{crit}}
\newcommand{\popt}{\pi_{det}^*}

In this section, we provide additional simulations on examples with suboptimal critical points. The two state example from \cref{running_example} is an example of state-aggregation and here we will give examples of the other applications from \cref{applications}. Our first two simulations ``Number Matching with Independent Agents'' and ``Button Press with Decentralized Agents'' are examples of independent multi-agents and decentralized multi-agents respectively (in the two agent case, decentralized and group decentralized agents are identical).

Then, we will show the surprising effectiveness of these methods in the fully observable setting.
Naively, it seems that by the gradient dominance bound shown in \cite{agarwal2019reinforcement}, that there should not be any suboptimal critical points when the setting is fully observable. However, this bound depends on the factor $||d_\mu^{\pi^*}/d_\mu^{ \pi}||_{\infty}$ or $||d_\mu^{\pi^*}/\mu||_{\infty}$ which can easily be infinite when the support of $\mu$ or $d_\mu^{\pi}$ is off even by a single state from $d_\mu^{\pi^*}$ such as in the examples we will provide. This invalidates the bound and easily creates suboptimal critical points as we will demonstrate. These critical points can then be escaped using our k-step methods as predicted by \cref{crit_guarantee}.

\subsection{Number Matching with Independent Agents}
\label{num_match}

\begin{figure}[h]
\centering
\begin{center}
\includegraphics[width=1.00\textwidth]{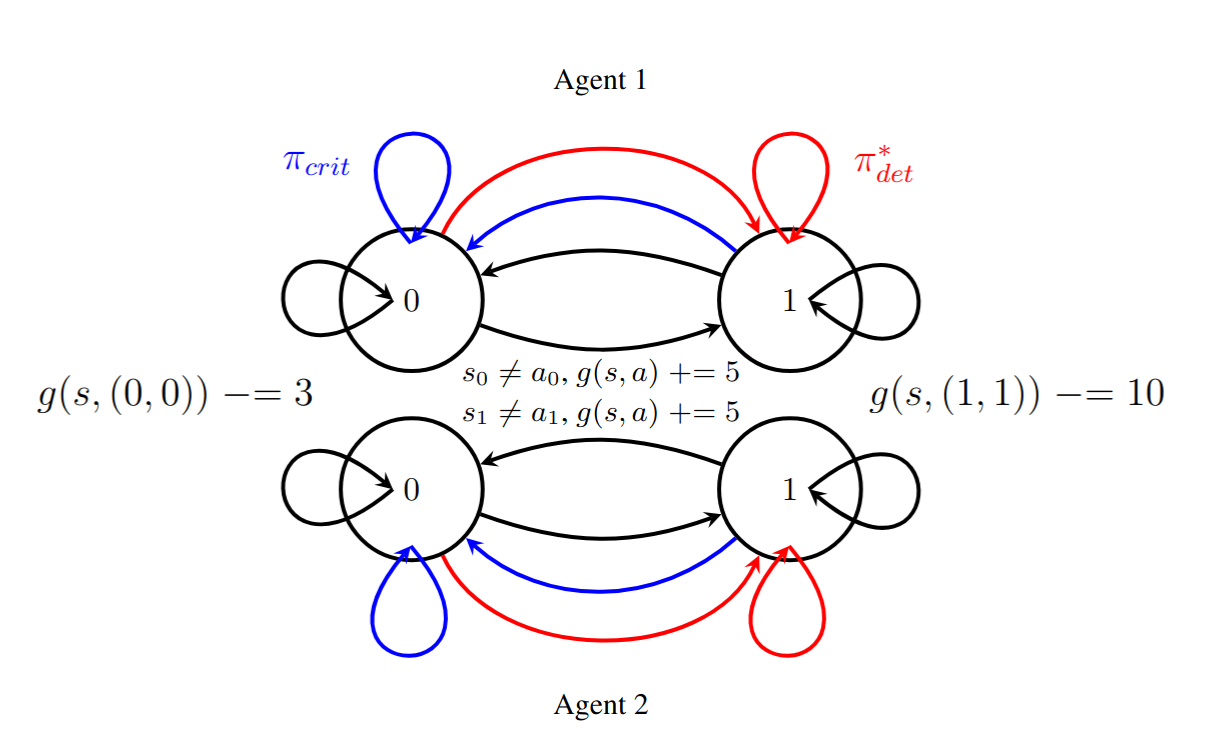}

\end{center}
\caption{Each agent independently transitions between states $0$ and $1$ and are rewarded with cost reductions when the numbers match. Switching their numbers incur penalties. The optimal policy in red $\popt$ is to match the number 1. The suboptimal critical point $\psub$ indicated in blue matches the number zero.}
\label{fig:exp1}
\end{figure}

In this simulation, two agents $\mathcal N = \{0,1\}$ will either be in a state of $0$ or $1$ with a joint state space of $\mathcal S = \{(s_1, s_2)\} = \{(0,0), (0,1), (1,0), (1,1)\}$. Each agent can take an action of either $0$ or $1$ which will determine which state they will transition to next deterministically with a joint action space of $\mathcal A = \{(a_1, a_2)\} = \{(0,0), (0,1), (1,0), (1,1)\}$ (e.g. taking action $(1,1)$ will transition the agents to state $(1,1)$). When the agents take the joint action $(0,0)$ they obtain a cost reduction of $-3$ and for joint action $(1,1)$ will obtain a higher cost reduction of $-10$. However, there will be a penalty of $+5$ for each agent that has $s_i \neq a_i$ (the agent chooses to switch its state). We will use a discount factor of $\gamma = 0.9$ and initial state distribution $\mu = [\mu(0,0), \mu(0,1), \mu(1,0), \mu(1,1)] = [0.05, 0.37, 0.37, 0.21]$.

Notice that in the fully observable case, since the support of the initial state distribution $\mu$ is over all states, we would be theoretically guaranteed to not have suboptimal critical points with trivial parametrization due to the gradient dominance condition in \cite{agarwal2019reinforcement}. However, as we will see, a suboptimal critical point will be created by considering a partial observability structure.

In terms of the policy restriction, agents will be independent so they must take an action that only depends on their own state. In otherwords, $\Pi_{det} = \{(\pi_{det}^1, \pi_{det}^2)\}$ where $\pi_{det}^i: s_i \rightarrow a_i$. $\tilde\Pi_{res}$ will then contain a distribution over $\Pi_{det}$ and we will assume a trivial parameterization which holds a weight for each element of $\Pi_{det}$.

We define the four individual deterministic policies available to each agent: $\pi_{\mathtt{a0}}$ always takes action $0$, $\pi_{\mathtt{a1}}$ always takes action $1$, $\pi_{\mathtt{st}}(s_i) = s_i$ stays in the current state, and $\pi_{\mathtt{fl}}(s_i) = 1 - s_i$ flips to the opposite state. There are $4 \times 4 = 16$ joint deterministic policies in $\Pi_{det}$.

The optimal policy $\popt = (\pi_{\mathtt{a1}}, \pi_{\mathtt{a1}})$ has both agents deterministically move to state $1$, yielding $J^{\popt, 1}(\mu) = \mu \cdot [-90, -95, -95, -100]^T = -95.8$. We claim that $\psub = (\pi_{\mathtt{a0}}, \pi_{\mathtt{a0}})$, which has both agents always move to state $0$, is a suboptimal $1$-step critical point with $J^{\psub, 1}(\mu) = \mu^T [-30, -25, -25, -20] = -24.2$.

To show $\psub$ is a $1$-step critical point we compute the $1$-step advantages $A^{\psub, 1}(s, \pi') = Q^{\psub, 1}(s, \pi') - J^{\psub}(s)$ for all $\pi' \in \Pi_{det}$ at each joint state $s$, using the discounted occupancy $d_\mu^{\psub} = [d_{\mu}^{\psub}(0,0), d_{\mu}^{\psub}(0,1), d_{\mu}^{\psub}(1,0), d_{\mu}^{\psub}(1,1)] = [0.905, 0.037, 0.037, 0.021]$. For the purposes of the table, we abbreviate $A^{\psub, 1}((s_1,s_2), \pi')$ as $A^1(s_1,s_2)$ and write $\bar{A}^1 = d_\mu^{\psub}\cdot [A^1(0,0), A^1(0,1), A^1(1,0), A^1(1,1)]^T$ for the weighted average.

\begin{center}
\footnotesize
\begin{tabular}{l rrrr r}
\toprule
Policy $\pi'$ & $A^1(0,0)$ & $A^1(0,1)$ & $A^1(1,0)$ & $A^1(1,1)$ & $\bar{A}^1$ \\
\midrule
$(\pi_{\mathtt{a0}},\pi_{\mathtt{a0}})$ [$\psub$] & $0.000$ & $0.000$ & $0.000$ & $0.000$ & $+0.0000$ \\
$(\pi_{\mathtt{a0}},\pi_{\mathtt{a1}})$ & $+12.500$ & $+2.500$ & $+12.500$ & $+2.500$ & $+11.9200$ \\
$(\pi_{\mathtt{a0}},\pi_{\mathtt{fl}})$ & $+12.500$ & $0.000$ & $+12.500$ & $0.000$ & $+11.7750$ \\
$(\pi_{\mathtt{a0}},\pi_{\mathtt{st}})$ & $0.000$ & $+2.500$ & $0.000$ & $+2.500$ & $+0.1450$ \\
$(\pi_{\mathtt{a1}},\pi_{\mathtt{a0}})$ & $+12.500$ & $+12.500$ & $+2.500$ & $+2.500$ & $+11.9200$ \\
$(\pi_{\mathtt{a1}},\pi_{\mathtt{a1}})$ [$\popt$] & $+12.000$ & $+2.000$ & $+2.000$ & $-8.000$ & $+10.8400$ \\
$(\pi_{\mathtt{a1}},\pi_{\mathtt{fl}})$ & $+12.000$ & $+12.500$ & $+2.000$ & $+2.500$ & $+11.4490$ \\
$(\pi_{\mathtt{a1}},\pi_{\mathtt{st}})$ & $+12.500$ & $+2.000$ & $+2.500$ & $-8.000$ & $+11.3110$ \\
$(\pi_{\mathtt{fl}},\pi_{\mathtt{a0}})$ & $+12.500$ & $+12.500$ & $0.000$ & $0.000$ & $+11.7750$ \\
$(\pi_{\mathtt{fl}},\pi_{\mathtt{a1}})$ & $+12.000$ & $+2.000$ & $+12.500$ & $+2.500$ & $+11.4490$ \\
$(\pi_{\mathtt{fl}},\pi_{\mathtt{fl}})$ & $+12.000$ & $+12.500$ & $+12.500$ & $0.000$ & $+11.7850$ \\
$(\pi_{\mathtt{fl}},\pi_{\mathtt{st}})$ & $+12.500$ & $+2.000$ & $0.000$ & $+2.500$ & $+11.4390$ \\
$(\pi_{\mathtt{st}},\pi_{\mathtt{a0}})$ & $0.000$ & $0.000$ & $+2.500$ & $+2.500$ & $+0.1450$ \\
$(\pi_{\mathtt{st}},\pi_{\mathtt{a1}})$ & $+12.500$ & $+2.500$ & $+2.000$ & $-8.000$ & $+11.3110$ \\
$(\pi_{\mathtt{st}},\pi_{\mathtt{fl}})$ & $+12.500$ & $0.000$ & $+2.000$ & $+2.500$ & $+11.4390$ \\
$(\pi_{\mathtt{st}},\pi_{\mathtt{st}})$  & $0.000$ & $+2.500$ & $+2.500$ & $-8.000$ & $+0.0170$ \\
\bottomrule
\end{tabular}
\end{center}

All weighted $1$-step advantages $\bar{A}^1$ are non-negative and indicating the one step costs are worse. By \cref{kpg_representation}, with the trivial parameterization the $1$-step policy gradient with respect to each $\pi' \in \Pi_{det}$ is proportional to $\bar{A}^1$, so every gradient direction is non-improving. Thus $\psub$ is a suboptimal $1$-step critical point.

Computing the $k$-step advantages of $\popt = (\pi_{\mathtt{a1}}, \pi_{\mathtt{a1}})$:

\begin{center}
\begin{tabular}{r rrrr r}
\toprule
$k$ & $A^k(0,0)$ & $A^k(0,1)$ & $A^k(1,0)$ & $A^k(1,1)$ & $\bar{A}^k$ \\
\midrule
$1$ & $+12.0000$ & $+2.0000$ & $+2.0000$ & $-8.0000$ & $+10.8400$ \\
$2$ & $+4.8000$ & $-5.2000$ & $-5.2000$ & $-15.2000$ & $+3.6400$ \\
$\mathbf{3}$ & $\mathbf{-1.6800}$ & $\mathbf{-11.6800}$ & $\mathbf{-11.6800}$ & $\mathbf{-21.6800}$ & $\mathbf{-2.8400}$ \\
$4$ & $-7.5120$ & $-17.5120$ & $-17.5120$ & $-27.5120$ & $-8.6720$ \\
$5$ & $-12.7608$ & $-22.7608$ & $-22.7608$ & $-32.7608$ & $-13.9208$ \\
$10$ & $-32.1057$ & $-42.1057$ & $-42.1057$ & $-52.1057$ & $-33.2657$ \\
$25$ & $-54.2568$ & $-64.2568$ & $-64.2568$ & $-74.2568$ & $-55.4168$ \\
\bottomrule
\end{tabular}
\end{center}

At $k = 3$ the weighted average $\bar{A}^k$ first becomes negative, so $k_{\mathrm{esc}} = 3$. A $k$-step policy gradient with $k \geq 3$ will descend toward $\popt$, whereas $k \in \{1, 2\}$ remains stuck at $\psub$. The penalty for switching state dominates the $1$-step and $2$-step horizons; by $k = 3$ the repeated lower cost at state $(1,1)$ outweighs the switching cost.

\subsection{Button Press with Decentralized Agents}
\label{button}

\begin{figure}[h]
\centering

\includegraphics[width=0.80\textwidth]{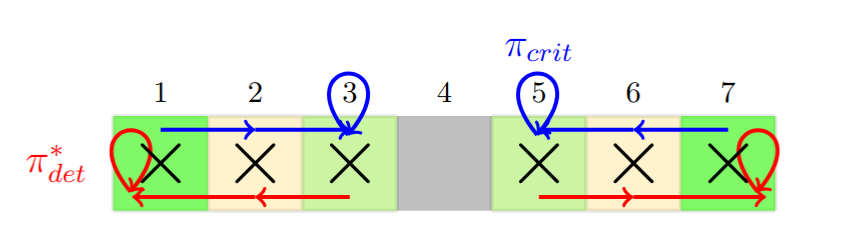}
\caption{A left agent in $\{1,2,3\}$ and a right agent in $\{5,6,7\}$ are separated by a wall at position $4$. When at positions $(3,5)$ staying at that position interpreted as jointly ``pushing a button'' will reduce the cost of agents ($-5$) with a penalty if either decide to leave ($+30$). Similarly there are a set of ``buttons'' at $(1,7)$ but with a better cost reduction $-18$. Indicated in red is the optimal deterministic policy $\popt$ to bring the agents towards the higher reward button states $(1,7)$. There is a suboptimal deterministic policy $\psub$ which brings both agents to the center button state $(3,5)$ and is a critical point. Crosses mark the support of the initial state distribution.}
\label{fig:exp5}
\end{figure}
This experiment illustrates that a decentralized policy restriction alone can create suboptimal critical points even when the initial distribution $\mu$ has full support. In the fully observable setting, full-support $\mu$ together with the gradient dominance bound of \cref{crit_guarantee} would guarantee that there are no suboptimal critical points. Here, the decentralized constraint is the source of the critical point.

A left agent occupies $s_L \in \{1, 2, 3\}$ and a right agent occupies $s_R \in \{5, 6, 7\}$, separated by a wall at position $4$ so the agents cannot cross. Each agent can move to an adjacent position or stay and the boundary positions are clamped. The immediate cost depends only on the joint state with $g(3,5) = -5$ (center button, if staying), $g(1,7) = -18$ (high-reward button, if staying), leaving either button incurs a positive cost ($-5 + 30 = +25$ or $-18 + 30 = +12$ respectively), and all other states have cost $0$. We use $\gamma = 0.9$ and $\mu$ uniform over all $9$ joint states.

Agents are decentralized with visibility $\mathcal V = 2$: at every joint state except $(3,5)$, neither agent can observe the other, so each agent's action may only depend on its own position. At the sole mutually visible state $(3,5)$ ($|3-5|=2 = \mathcal V$), both agents observe the full joint state and may independently choose their action there. The left agent's policy therefore has four independent components: $\pi^L(1) \in \{1,2\}$ and $\pi^L(2) \in \{1,2,3\}$ applied at all out-of-view states with those positions, $\pi^L(3) \in \{2,3\}$ applied only at the out-of-view states $(3,6)$ and $(3,7)$, and a separate visible-state action $\pi^L_{\mathrm{vis}} \in \{2,3\}$ applied at $(3,5)$ independently. This gives $2 \times 3 \times 2 \times 2 = 24$ left-agent policies and by symmetry the right agent also has $24$. There are $24 \times 24 = 576$ joint deterministic policies in $\Pi_{det}$.

The optimal policy $\popt$ sends each agent toward its end of the grid and stays at the high-reward button: $J^{\popt}(\mu) = -152.22$. We claim that $\psub$, which sends the left agent toward $3$ and the right agent toward $5$ and stays at the center button, is a suboptimal $1$-step critical point with $J^{\psub}(\mu) = -41.72$.

Under $\psub$, occupancy concentrates at $(3, 5)$. The full discounted occupancy is
  $d_\mu^{\psub}(3,5) = 0.861,
  d_\mu^{\psub}(2,5) = d_\mu^{\psub}(3,6) =  0.031,
  d_\mu^{\psub}(2,6) =  0.021,$
and $d_\mu^{\psub}(s) = 0.011$ for $s \in \{(1,5),(1,6),(1,7),(2,7),(3,7)\}$.

All $576$ joint deterministic policies $\pi' \in \Pi_{det}$ satisfy $d_\mu^{\psub} \cdot \bar{A}^{\psub,1}(\cdot, \pi')^T \geq 0$ (verified numerically, omitted for brevity).

Below, we provide the $k$-step advantages of $\popt$:

\begin{center}
\resizebox{\linewidth}{!}{
\begin{tabular}{r rrr rrr rrr r}
\toprule
$k$
  & $A^k(1,5)$ & $A^k(1,6)$ & $A^k(1,7)$
  & $A^k(2,5)$ & $A^k(2,6)$ & $A^k(2,7)$
  & $A^k(3,5)$ & $A^k(3,6)$ & $A^k(3,7)$
  & $\bar{A}^k$ \\
\midrule
$1$
  & $+4.050$ & $+14.850$ & $-15.150$
  & $+8.550$ & $+19.350$ & $+14.850$
  & $+34.500$ & $+8.550$ & $+4.050$
  & $+30.9005$ \\
$2$
  & $+17.415$ & $+1.215$ & $-28.785$
  & $+21.915$ & $+5.715$ & $+1.215$
  & $+51.915$ & $+21.915$ & $+17.415$
  & $+46.2830$ \\
$3$
  & $+5.143$ & $-11.057$ & $-41.057$
  & $+9.643$ & $-6.557$ & $-11.057$
  & $+39.643$ & $+9.643$ & $+5.143$
  & $+34.0115$ \\
$4$
  & $-5.901$ & $-22.101$ & $-52.101$
  & $-1.401$ & $-17.601$ & $-22.101$
  & $+28.599$ & $-1.401$ & $-5.901$
  & $+22.9672$ \\
$5$
  & $-15.841$ & $-32.041$ & $-62.041$
  & $-11.341$ & $-27.541$ & $-32.041$
  & $+18.659$ & $-11.341$ & $-15.841$
  & $+13.0272$ \\
$6$
  & $-24.787$ & $-40.987$ & $-70.987$
  & $-20.287$ & $-36.487$ & $-40.987$
  & $+9.713$ & $-20.287$ & $-24.787$
  & $+4.0813$ \\
$\mathbf{7}$
  & $\mathbf{-32.838}$ & $\mathbf{-49.038}$ & $\mathbf{-79.038}$
  & $\mathbf{-28.338}$ & $\mathbf{-44.538}$ & $\mathbf{-49.038}$
  & $\mathbf{+1.662}$ & $\mathbf{-28.338}$ & $\mathbf{-32.838}$
  & $\mathbf{-3.9700}$ \\
$8$
  & $-40.084$ & $-56.284$ & $-86.284$
  & $-35.584$ & $-51.784$ & $-56.284$
  & $-5.584$ & $-35.584$ & $-40.084$
  & $-11.2162$ \\
\bottomrule
\end{tabular}}
\end{center}

At $k = 7$ the weighted average $\bar{A}^k$ first becomes negative, so $k_{\mathrm{esc}} = 7$. The center button $(3,5)$ maintains a positive advantage ($A^k(3,5) > 0$) through $k = 6$, pushing the weighted average in the wrong direction until the high-reward button advantage at $(1,7)$ and intermediate states accumulates sufficiently.

\subsection{Fully Observable Moat Cross}
\label{moat}

\begin{figure}[h]
\centering
\includegraphics[width = 0.8\textwidth]{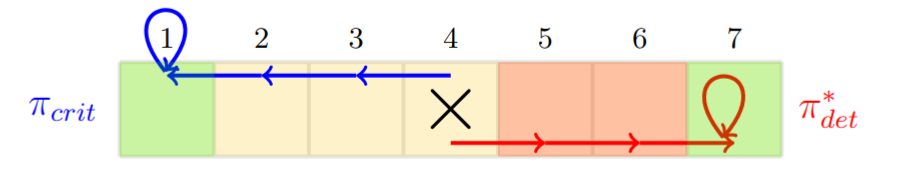}

\caption{A $7$-long grid world where the agent starts at state $4$ ($\mu = \delta_4$). The costs in the diagram are $g(1)=-1$, $g(5)=g(6)=+3$ (moat), $g(7)=-20$. The optimal policy in red $\popt$ will overcome the penalty from the moat and move to the higher reward square at state 7. A suboptimal critical point will be present at $\psub$ which goes directly to state $1$.}
\label{fig:exp3}
\end{figure}

With this experiment, we demonstrate the benefits of our theory in a fully observable setting. Recall that in the unrestricted policy class setting there exists a gradient dominance bound in \cite{agarwal2019reinforcement}. However, this bound depends on the factor $||d_\mu^{\pi^*}/d_\mu^{ \pi}||_{\infty}$ or $||d_\mu^{\pi^*}/\mu||_{\infty}$ which will be infinite in this example (since the support of $\mu$ and $d_{\mu}^{\pi}$ does not subsume the support of $d_{\mu}^{\pi^*}$). Therefore, suboptimal critical points may emerge even when the policy class is unrestricted.

In this single-agent experiment, the state space is a grid world that has $7$ states in a row $\mathcal{S} = \{1, \ldots, 7\}$ with the agent starting at state $4$ ($\mu = \delta_4$). At each step the agent chooses to move left ($a = -1$), stay ($a = 0$), or move right ($a = +1$), and transitions to the adjacent state deterministically (clamped at boundaries). The per-step costs are $g(1) = -1$, $g(5) = g(6) = +3$ (moat), $g(7) = -20$ and all other states have $g = 0$. We use $\gamma = 0.9$.

The suboptimal policy $\psub$ always moves left and remains at state $1$ with $J^{\psub}(\delta_4) = -7.29$. The optimal policy $\popt$ always moves right, crossing the moat (states $5,6$) to stay at state $7$ with $J^{\popt}(\delta_4) = -140.67$.
The discounted occupancy under $\psub$ is
  $d_{\delta_4}^{\psub}(1) = 0.729,\quad
  d_{\delta_4}^{\psub}(2) = 0.081,\quad
  d_{\delta_4}^{\psub}(3) = 0.090,\quad
  d_{\delta_4}^{\psub}(4) = 0.100,\quad
  d_{\delta_4}^{\psub}(s) = 0 \text{ for } s \in \{5,6,7\}$, making the gradient dominance bound in \cite{agarwal2019reinforcement} vacuous since $\popt$ visits states unreachable under $\psub$.

At every support state $s \in \{1,2,3,4\}$, every action $a \in \{-1,0,+1\}$ has $A^{\psub,1}(s,a) \geq 0$, confirming that $\psub$ is a $1$-step critical point:

\begin{center}
\begin{tabular}{lll rr}
\toprule
State $s$ & $J^{\psub}(s)$ & Action $a$ & $Q^{\psub}(s,a)$ & $A^{\psub}(s,a)$ \\
\midrule
\multirow{3}{*}{$s=1$} & \multirow{3}{*}{$-10.00$}
  & $a{=}{-1}$ [$\psub$] & $-10.000$ & $0.000$ \\
  & & $a{=}0$ & $-10.000$ & $0.000$ \\
  & & $a{=}{+1}$ & $-9.100$ & $+0.900$ \\
\midrule
\multirow{3}{*}{$s=2$} & \multirow{3}{*}{$-9.00$}
  & $a{=}{-1}$ [$\psub$] & $-9.000$ & $0.000$ \\
  & & $a{=}0$ & $-8.100$ & $+0.900$ \\
  & & $a{=}{+1}$ & $-7.290$ & $+1.710$ \\
\midrule
\multirow{3}{*}{$s=3$} & \multirow{3}{*}{$-8.10$}
  & $a{=}{-1}$ [$\psub$] & $-8.100$ & $0.000$ \\
  & & $a{=}0$ & $-7.290$ & $+0.810$ \\
  & & $a{=}{+1}$ & $-6.561$ & $+1.539$ \\
\midrule
\multirow{3}{*}{$s=4$} & \multirow{3}{*}{$-7.29$}
  & $a{=}{-1}$ [$\psub$] & $-7.290$ & $0.000$ \\
  & & $a{=}0$ & $-6.561$ & $+0.729$ \\
  & & $a{=}{+1}$ & $-3.205$ & $+4.085$ \\
\bottomrule
\end{tabular}
\end{center}

The positive $A^{\psub}$ for rightward actions reflects the fact that the moat states $5,6$ have low value under $\psub$: transitioning into the moat is immediately costly and the future value $J^{\psub}(5) = -3.56$ is poor. Computing the $k$-step advantages of $\popt$:

\begin{center}
\begin{tabular}{r rrrr r}
\toprule
$k$ & $A^k(1)$ & $A^k(2)$ & $A^k(3)$ & $A^k(4)$ & $\bar{A}^k$ \\
\midrule
$1$  & $+0.900$ & $+1.710$ & $+1.539$ & $+4.085$ & $+1.342$ \\
$2$  & $+2.439$ & $+3.095$ & $+5.216$ & $+9.824$ & $+3.481$ \\
$3$  & $+3.686$ & $+6.404$ & $+10.381$ & $-2.294$ & $+3.910$ \\
$4$  & $+6.664$ & $+11.053$ & $-0.526$ & $-15.403$ & $+4.165$ \\
$5$  & $+10.847$ & $+1.237$ & $-12.324$ & $-27.201$ & $+4.179$ \\
$\mathbf{6}$  & $\mathbf{+2.013}$ & $\mathbf{-9.381}$ & $\mathbf{-22.942}$ & $\mathbf{-37.819}$ & $\mathbf{-5.139}$ \\
$7$  & $-7.543$ & $-18.937$ & $-32.498$ & $-47.375$ & $-14.695$ \\
$10$ & $-30.851$ & $-42.245$ & $-55.805$ & $-70.682$ & $-38.003$ \\
\bottomrule
\end{tabular}
\end{center}

At $k = 6$ the weighted average $\bar{A}^k$ first becomes negative, so $k_{\mathrm{esc}} = 6$. Note that state $4$ escapes as early as $k = 3$ (its per-state advantage turns negative), but the high occupancy mass at state $1$ ($d_{\delta_4}^{\psub}(1) = 0.729$) keeps $\bar{A}^k$ positive through $k = 5$.

\subsection{Fully Observable Two Path Navigation}
\label{two_path}

\begin{figure}[h]
\centering
\includegraphics[width = 0.4\textwidth]{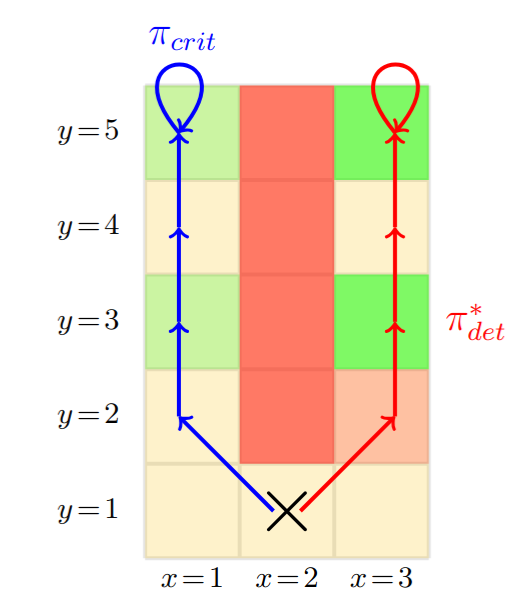}
\caption{A $3\times 5$ grid where the agent starts at $(2,1)$ ($\mu = \delta_{(2,1)}$) and is forced upward each step. The optimal policy indicated in red $\popt$ moves right incurring a minor penalty of $+1$ at $(3,2)$ and then passes through a high cost reduction intermediate state $(3,3)$ of $-15$ and is absorbing at $(3, 5)$ which gives $-20$. The suboptimal policy $\pi_{sub}$ moves left passing through the minor cost reduction of $-2$ at $(1,3)$ and absorbs at a $(1,5)$ which gives a minor cost reduction of $-5$.}
\label{fig:exp4}
\end{figure}

Similar to the previous experiment, we will consider another fully observable example where a suboptimal critical point emerges.

In this single-agent, fully observable experiment the state space is $\mathcal{S} = \{1,2,3\} \times \{1,\ldots,5\}$ with the agent starting at $(2,1)$ ($\mu = \delta_{(2,1)}$). The row coordinate increases by $1$ deterministically each step regardless of the action; the agent's action $a \in \{-1, 0, +1\}$ controls horizontal movement (clipped to columns $\{1,2,3\}$). The per-step cost is $g(2, y) = +10$ for $y \geq 2$ (the crater column), $g(1,3) = -2$, $g(1,5) = -5$, $g(3,3) = -15$, $g(3,5) = -20$ $g(3,2) = +1$, and $g = 0$ elsewhere. Both $(1,5)$ and $(3,5)$ are absorbing. We use $\gamma = 0.9$.

The suboptimal policy $\psub$ always moves left, following $(2,1) \to (1,2) \to \cdots \to (1,5)$, with $J^{\psub}(\delta_{(2,1)}) = -34.43$. The optimal policy $\popt$ always moves right, following $(2,1) \to (3,2) \to \cdots \to (3,5)$, with $J^{\popt}(\delta_{(2,1)}) = -142.47$. The support of $d_{\delta_{(2,1)}}^{\psub}$ is $\{(2,1),(1,2),(1,3),(1,4),(1,5)\}$. Most states on the right path lie entirely outside this support, making the gradient dominance bound vacuous.

The discounted occupancy under $\psub$ is
  $d_{\delta_{(2,1)}}^{\psub}(2,1) = 0.100,
  d_{\delta_{(2,1)}}^{\psub}(1,2) = 0.090,
  d_{\delta_{(2,1)}}^{\psub}(1,3) = 0.081,
  d_{\delta_{(2,1)}}^{\psub}(1,4) = 0.073,
  d_{\delta_{(2,1)}}^{\psub}(1,5) = 0.656$
with $d_{\delta_{(2,1)}}^{\psub}(s) = 0$ elsewhere.

At every support state, all three actions have $A^{\psub,1}(s,a) \geq 0$. The right-moving actions enter the crater (or hit the $+1$ penalty before moving into the crater), and the left-moving or stay actions either are already following $\psub$ or enters the crater. The full table of advantages across the support of $d_{\delta_{(2,1)}}^{\psub}$ is shown below.

\begin{center}
\begin{tabular}{lll rr}
\toprule
State $s$ & $J^{\psub}(s)$ & Action $a$ & $Q^{\psub}(s,a)$ & $A^{\psub}(s,a)$ \\
\midrule
\multirow{3}{*}{$(2,1)$} & \multirow{3}{*}{$-34.425$}
  & $a{=}{-1}$ [$\psub$] & $-34.425$ & $0.000$ \\
  & & $a{=}0$ & $-25.425$ & $+9.000$ \\
  & & $a{=}{+1}$ [$\popt$] & $-23.805$ & $+10.620$ \\
\midrule
\multirow{3}{*}{$(1,2)$} & \multirow{3}{*}{$-38.250$}
  & $a{=}{-1}$ [$\psub$] & $-38.250$ & $0.000$ \\
  & & $a{=}0$ & $-38.250$ & $0.000$ \\
  & & $a{=}{+1}$ & $-27.450$ & $+10.800$ \\
\midrule
\multirow{3}{*}{$(1,3)$} & \multirow{3}{*}{$-42.500$}
  & $a{=}{-1}$ [$\psub$] & $-42.500$ & $0.000$ \\
  & & $a{=}0$ & $-42.500$ & $0.000$ \\
  & & $a{=}{+1}$ & $-33.500$ & $+9.000$ \\
\midrule
\multirow{3}{*}{$(1,4)$} & \multirow{3}{*}{$-45.000$}
  & $a{=}{-1}$ [$\psub$] & $-45.000$ & $0.000$ \\
  & & $a{=}0$ & $-45.000$ & $0.000$ \\
  & & $a{=}{+1}$ & $-31.500$ & $+13.500$ \\
\midrule
\multirow{3}{*}{$(1,5)$} & \multirow{3}{*}{$-50.000$}
  & $a{=}{-1}$ [$\psub$] & $-50.000$ & $0.000$ \\
  & & $a{=}0$ & $-50.000$ & $0.000$ \\
  & & $a{=}{+1}$ & $-36.500$ & $+13.500$ \\
\bottomrule
\end{tabular}
\end{center}

Computing the $k$-step advantages of $\popt$ (all actions set to right, $a = +1$):

\begin{center}
\begin{tabular}{r rrrrr r}
\toprule
$k$ & $A^k(1,2)$ & $A^k(1,3)$ & $A^k(1,4)$ & $A^k(1,5)$ & $A^k(2,1)$ & $\bar{A}^k$ \\
\midrule
$1$ & $+10.800$ & $+9.000$ & $+13.500$ & $+13.500$ & $+10.620$ & $+12.605$ \\
$2$ & $+21.735$ & $+7.785$ & $+12.285$ & $+12.285$ & $-2.340$ & $+11.309$ \\
$3$ & $+9.706$ & $-4.244$ & $+0.256$ & $+0.256$ & $+0.212$ & $+0.738$ \\
$\mathbf{4}$ & $\mathbf{-1.119}$ & $\mathbf{-15.069}$ & $\mathbf{-10.569}$ & $\mathbf{-10.569}$ & $\mathbf{-10.614}$ & $\mathbf{-10.088}$ \\
$5$ & $-10.862$ & $-24.812$ & $-20.312$ & $-20.312$ & $-20.357$ & $-19.831$ \\
$10$ & $-46.771$ & $-60.721$ & $-56.221$ & $-56.221$ & $-56.266$ & $-55.740$ \\
\bottomrule
\end{tabular}
\end{center}

At $k = 4$ the weighted average $\bar{A}^k$ first becomes negative, so $k_{\mathrm{esc}} = 4$. Four steps of lookahead are required to see past the crater penalty and notice that $(3,5)$ offers substantially lower long-run cost than $(1,5)$.

\newpage

\section{Proofs}
\label{proofs_section}
\subsection{$k$-step Policy Gradient Theorem}
\label{kpg_proof}
Below we prove the $k$-step Policy Gradient theorem in \cref{kpg_representation}. The proof mirrors the standard policy gradient theorem proof from \cite{agarwal2019reinforcement}.

\textbf{Proof for \cref{kpg_representation}}
\begin{proof}
\begin{align}
    &\nabla_\alpha J^{\tilde\pi_\alpha,k}(s_0)\label{k_pg_0}\\
    & = \nabla_\alpha \mathbb E_{\pi_{det} \sim \tilde\pi_\alpha}[Q^{\tilde\pi_\alpha, k} (s_0, \pi_{det})]\\
    &= \nabla_\alpha \sum_{\pi_{det}} \tilde\pi_\alpha(\pi_{det}) Q^{\tilde\pi_\alpha, k}(s_0, \pi_{det})\\
    &= \sum_{\pi_{det}}  (\nabla_\alpha\tilde\pi_\alpha(\pi_{det})) Q^{\tilde\pi_\alpha, k}(s_0, \pi_{det})+ \sum_{\pi_{det}} \tilde\pi_\alpha(\pi_{det})\nabla_\alpha  Q^{\tilde\pi_\alpha, k}(s_0, \pi_{det})\label{k_pg_1}\\
    &= \sum_{\pi_{det}}  (\tilde\pi_\alpha(\pi_{det})\nabla_\alpha \log\tilde\pi_\alpha(\pi_{det})) Q^{\tilde\pi_\alpha, k}(s_0, \pi_{det})+ \gamma^k\sum_{\pi_{det}} \tilde\pi_\alpha(\pi_{det}) \mathbb E_{\tau^k \sim \pi_{det}\lvert_{s_0}}[\nabla_\alpha J^{\tilde\pi_\alpha, k}(s_k)]\label{k_pg_2}\\
    &= \mathbb E_{\pi_{det} \sim \tilde\pi_\alpha}[ Q^{\tilde\pi_\alpha, k}(s_0, \pi_{det})\nabla_\alpha \log\tilde\pi_\alpha(\pi_{det})]+ \gamma^k\sum_{\pi_{det}} \tilde\pi_\alpha(\pi_{det}) \mathbb E_{\tau^k \sim \pi_{det}\lvert_{s_0}}[\nabla_\alpha J^{\tilde\pi_\alpha, k}(s_k)]\label{k_pg_2_5}\\
    &= \ldots\\
    &=\sum_{m = 0}^\infty \gamma^{mk}\sum_{s\in \mathcal S} P(s_{mk} = s\lvert s_0, \tilde\pi_\alpha,k)\sum_{\pi_{det}}\tilde\pi_\alpha(\pi_{det}) Q^{\tilde\pi_\alpha, k}(s_{mk}, \pi_{det})\nabla_\alpha \log\tilde\pi_\alpha(\pi_{det})\label{k_pg_3}\\
    & = \frac{1}{1 - \gamma^{k}} \mathbb E_{s \sim d_{s_0}^{\tilde\pi_\alpha, k}}\mathbb E_{\pi_{det}\sim \tilde\pi_\alpha}[ Q^{\tilde\pi_\alpha, k}(s, \pi_{det})\nabla_\alpha \log\tilde\pi_\alpha(\pi_{det})]\label{k_pg_4}
\end{align}
Chain rule is used in  \cref{k_pg_1}. For   \cref{k_pg_2}, $\nabla_\alpha  Q^{\tilde\pi_\alpha, k}(s_0, \pi_{det}) = \gamma^k E_{\tau^k \sim \pi_{det}\lvert_{s_0}}[\nabla_\alpha J^{\tilde\pi_\alpha, k}(s_k)]$ because the first $k$ timesteps are not dependent on $\alpha$. In \cref{k_pg_2_5}, $\tau^k\sim \pi_{det}\lvert_{s_0}$ refers to the trajectory generated by $\pi_{det}$ starting at $s_0$ up to timestep $k$. \cref{k_pg_3} is obtained through recursive substitution of \cref{k_pg_2_5}. Lastly, \cref{k_pg_4} is by the definition our $k$-step discounted state occupancy measure $d_{s_0}^{\tilde\pi_\alpha, k}(s) = (1 - \gamma^k)\sum_{m = 0}^\infty \gamma^{mk} P(s_{mk} = s\lvert s_0, \tilde\pi_\alpha, k)$.

Taking expectation with respect to $s_0 \sim \mu$ gives \cref{kpg_representation} as desired.
\end{proof}

\subsection{Main Results}
Recall from \cref{proof_sketch} that the proofs of the main results will be split into two parts. In part 1), we will show an upper bound for the directional derivative which will prove \cref{crit_guarantee} and will be used in part 2). We will then prove in Part 2), \cref{projected_gd_bound} and \cref{mirror_bound}.
\subsubsection{Part 1) Approximate Gradient Dominance}
\label{part1_proof}
Before proving \cref{crit_guarantee}, we will show the following bound on the directional derivative
\begin{lemma}[Approximate Gradient Dominance]
 $$\nabla_{\tilde\pi' - \tilde\pi} J^{\tilde\pi,k}(\mu) \leq  \frac{1}{1 - \gamma^k}   (J^{\tilde\pi',k}(\mu) -  J^{\tilde\pi,k}(\mu)) + \frac{6\gamma^k}{(1 - \gamma^k)(1 - \gamma)} g_{max}$$
 \label[lemma]{closed_form}
\end{lemma}
\begin{proof}

Using the definition of the directional derivative, and a $k$-step variant of the performance difference lemma proved in \cref{pd_lemma} we obtain:
\begin{align}
    & \nabla_{\tilde\pi' - \tilde\pi} J^{\tilde\pi,k}(\mu)= \lim_{\theta \rightarrow 0} \frac{J^{\tilde\pi + \theta(\tilde\pi' - \tilde\pi),k}(\mu) - J^{\tilde\pi,k}(\mu)}{\theta} \\
    &= \lim_{\theta \rightarrow 0} \frac{J^{(1 - \theta)\tilde\pi + \theta\tilde\pi',k}(\mu) - J^{\tilde\pi,k}(\mu)}{\theta} \\
    &=  \lim_{\theta \rightarrow 0} \frac{1}{1 - \gamma^k} \frac{\mathbb E_{s \sim d^{(1 - \theta)\tilde\pi + \theta\tilde\pi',k}_\mu} [Q^{\tilde\pi,k}(s, \theta\tilde\pi' + (1 - \theta) \tilde\pi) - J^{\tilde\pi,k}(s)]}{\theta}\label{crit_0}\\
    &=  \lim_{\theta \rightarrow 0} \frac{1}{1 - \gamma^k} \frac{\mathbb E_{s \sim d^{(1 - \theta)\tilde\pi + \theta\tilde\pi',k}_\mu} [\theta Q^{\tilde\pi,k}(s, \tilde\pi') +  (1 - \theta ) Q^{\tilde\pi,k}(s, \tilde\pi) - J^{\tilde\pi,k}(s)]}{\theta}\label{crit_1}\\
    &=  \lim_{\theta \rightarrow 0} \frac{1}{1 - \gamma^k} \frac{\mathbb E_{s \sim d^{(1 - \theta)\tilde\pi + \theta\tilde\pi',k}_\mu} [\theta Q^{\tilde\pi,k}(s, \tilde\pi')   - \theta J^{\tilde\pi,k}(s)]}{\theta}\\
    &=  \lim_{\theta \rightarrow 0} \frac{1}{1 - \gamma^k} \mathbb E_{s \sim d^{(1 - \theta)\tilde\pi + \theta\tilde\pi',k}_\mu} [ Q^{\tilde\pi,k}(s, \tilde\pi')   -  J^{\tilde\pi,k}(s)]\\
    &=  \frac{1}{1 - \gamma^k} \mathbb E_{s \sim d^{\tilde\pi,k}_\mu} [ Q^{\tilde\pi,k}(s, \tilde\pi')   -  J^{\tilde\pi,k}(s)]
    \end{align}
    In \cref{crit_1}, we use an affine property of the $k$-step $Q$-function in the correlated policy argument.  This can be shown by viewing $\theta\tilde\pi' + (1 - \theta) \tilde\pi$ as a policy where $\tilde\pi'$ or $\tilde\pi$ is chosen according to an independent indicator $I_\theta$. Conditioning on whether $\pi$ or $\pi'$ is chosen as follows $\mathbb E_{\pi_{det}\sim \theta\tilde\pi' + (1 - \theta) \tilde\pi}[Q^{\tilde\pi,k}(s, \pi_{det})] = \mathbb E_{I_\theta}[I_\theta \mathbb E_{\pi_{det}\sim \tilde\pi'}[Q^{\tilde\pi,k}(s, \pi_{det})] + (1 - I_\theta) E_{\pi_{det}\sim \tilde\pi}[ Q^{\tilde\pi,k}(s, \pi_{det})]] = \theta \mathbb E_{\pi_{det}\sim \tilde\pi'}[Q^{\tilde\pi,k}(s, \pi_{det})] +  (1 - \theta )E_{\pi_{det}\sim \tilde\pi}[ Q^{\tilde\pi,k}(s, \pi_{det})]$.

    This hints at a major theoretical reason why we use correlated policies in this work. The direction $\theta\tilde\pi' + (1 - \theta) \tilde\pi$ is still a correlated policy but stochastically interpolating state independent policies of the form $\pi:\mathcal S\rightarrow \Delta(\mathcal A)$ may result in correlations that cannot be expressed with these state independent policies.

    Next, we will use  $\sum_{s\in \mathcal S}\lvert d_\mu^{\tilde\pi,k}(s) - \mu(s) \rvert \leq 2\gamma^k$ proved in \cref{dist_close} to change the expectation to be over $\mu$ instead of $d_\mu^{\tilde \pi,k}$ as follows:
    \begin{align}
    \frac{1}{1 - \gamma^k} &\mathbb E_{s \sim d^{\tilde\pi,k}_\mu} [ Q^{\tilde\pi,k}(s, \tilde\pi')   -  J^{\tilde\pi,k}(s)]\\
    &\leq  \frac{1}{1 - \gamma^k} \mathbb E_{s \sim \mu} [ Q^{\tilde\pi,k}(s, \tilde\pi')   -  J^{\tilde\pi,k}(s)] +  \frac{1}{1 - \gamma^k} \sum_{s \in \mathcal S} \lvert \mu(s) - d_\mu^{\tilde\pi,k}(s)\rvert \lvert Q^{\tilde\pi,k}(s, \tilde\pi')   -  J^{\tilde\pi,k}(s)\rvert\\
    &\leq  \frac{1}{1 - \gamma^k} \mathbb E_{s \sim \mu} [ Q^{\tilde\pi,k}(s, \tilde\pi')   -  J^{\tilde\pi,k}(s)] + \frac{4\gamma^k}{(1 - \gamma^k)(1 - \gamma)} g_{max}\\
    &=  \frac{1}{1 - \gamma^k} \mathbb E_{s \sim \mu} [  J^{\tilde\pi',k}(s) -  J^{\tilde\pi,k}(s) - (J^{\tilde\pi',k}(s) -  Q^{\tilde\pi,k}(s, \tilde\pi'))] + \frac{4\gamma^k}{(1 - \gamma^k)(1 - \gamma)} g_{max}\\
    &=  \frac{1}{1 - \gamma^k} \mathbb E_{s \sim \mu} [  J^{\tilde\pi',k}(s) -  J^{\tilde\pi,k}(s) - \gamma^k\mathbb E_{\tau^k \sim \tilde\pi'\lvert_{s}}[J^{\tilde\pi',k}(s_k) -  J^{\tilde\pi,k}(s_k)]] + \frac{4\gamma^k}{(1 - \gamma^k)(1 - \gamma)} g_{max}\\
    &\leq  \frac{1}{1 - \gamma^k} \mathbb E_{s \sim \mu} [  J^{\tilde\pi',k}(s) -  J^{\tilde\pi,k}(s)] - \gamma^k\mathbb E_{\tau^k \sim \tilde\pi'\lvert_{s}}[\frac{-2}{(1 - \gamma)}g_{max}] + \frac{4\gamma^k}{(1 - \gamma^k)(1 - \gamma)} g_{max}\\
    &\leq  \frac{1}{1 - \gamma^k}   (J^{\tilde\pi',k}(\mu) -  J^{\tilde\pi,k}(\mu)) + \frac{6\gamma^k}{(1 - \gamma^k)(1 - \gamma)} g_{max}
\end{align}
which gives the lemma statement as desired.

\end{proof}

Now to prove \cref{crit_guarantee}, notice for any critical point $\tilde\pi_{crit}$, the directional derivative $\nabla_{\tilde\pi' - \tilde\pi_{crit}} J^{\tilde\pi_{crit},k}(\mu)$ exists for any $\tilde\pi' \in \tilde\Pi_{res}$ by differentiability. Further, the assumption in \cref{assumption} implies the domain $\tilde\Pi_{res}$ is convex and therefore $(1 - \theta) \tilde \pi + \theta \tilde\pi'\in \tilde\Pi_{res}$.
Therefore, a small change of $\tilde\pi_{crit}$ along the direction of $\tilde \pi' - \tilde \pi_{crit}$ will keep it in $\tilde \Pi_{res}$.

So, all zero gradient critical points will have $\nabla_{\tilde\pi' - \tilde\pi_{crit}} J^{\tilde\pi_{crit},k}(\mu) = 0$ and local minima on the boundary will have $\nabla_{\tilde\pi' - \tilde\pi_{crit}} J^{\tilde\pi_{crit},k}(\mu) \geq 0$ for any $\tilde\pi'\in \Pi_{res}$ because otherwise, it would contradict $\pi_{crit}$ having a zero gradient (all directional derivatives zero) or being a local minima (all directional derivatives in a direction staying in $\tilde \Pi_{res}$ are $\geq 0$).

\textbf{Proof for \cref{crit_guarantee}}:
\begin{proof}
Using \cref{closed_form} and the argument above with $\tilde\pi^*= \delta(\pi^*_{det}) \in \Pi_{res}$ in place of $\tilde\pi'$ and $\tilde \pi_{crit}$ in place of $\tilde \pi$, this yields:

$0 \leq \nabla_{\tilde\pi^* - \tilde\pi_{crit}} J^{\tilde\pi_{crit},k}(\mu) \leq \frac{1}{1 - \gamma^k}   (J^{\pi^*_{det}}(\mu) -  J^{\tilde\pi_{crit},k}(\mu)) + \frac{6\gamma^k}{(1 - \gamma^k)(1 - \gamma)} g_{max}$

simplifying to

$  J^{\tilde\pi_{crit},k}(\mu) - J^{\pi^*_{det}}(\mu) \leq \frac{6\gamma^k}{1 - \gamma} g_{max}$

This gives \cref{crit_guarantee} when the the difference between $\mathbb E_{\pi_{det}\sim\tilde \pi_{crit}}[J^{\pi_{det}}(\mu)]$ and $J^{\tilde\pi_{crit},k}(\mu)$ is bounded as follows:
\begin{align}
    &\mathbb E_{\pi_{det}\sim\tilde \pi_{crit}}[J^{\pi_{det}}(\mu)] - J^{\pi_{det}^*}(\mu) = J^{\tilde\pi_{crit},k}(\mu)- J^{\pi_{det}^*}(\mu) + (\mathbb E_{\pi_{det}\sim\tilde \pi_{crit}}[J^{\pi_{det}}(\mu)] - J^{\tilde\pi_{crit},k}(\mu))\\
    &\leq J^{\tilde\pi_{crit},k}(\mu) -  J^{\pi^*_{det}}(\mu)
    + \frac{2\gamma^k}{1 - \gamma} g_{max}\label{crit_10}\\
    &\leq \frac{8\gamma^k}{1 - \gamma} g_{max}.
\end{align}
Here, \cref{crit_10} uses that both terms $\mathbb E_{\pi_{det}\sim\tilde \pi_{crit}}[J^{\pi_{det}}(\mu)]$ and $J^{\tilde\pi_{crit},k}(\mu)$ have the same distribution of states and actions in the first $k$ iterations.

\end{proof}

\subsubsection{Part 2.1) Projected Gradient Descent}
\label{part2_proj}
\textbf{Proof for \cref{projected_gd_bound}}:

\begin{proof}

Projected gradient descent is a special case of mirror descent with mirror map $\Phi(x) = \frac{1}{2}\|x\|_2^2$ which is $1$-strongly convex.

The Bregman divergence takes the form $D_{\Phi}(x, y) = \frac{1}{2}\|x\|_2^2 - \frac{1}{2}\|y\|_2^2 - y^T(x - y) = \frac{1}{2}\|x - y\|_2^2$.

Using \cref{mirror_bound}, we obtain
\begin{align}
\mathbb E_{\pi_{det}\sim\tilde \pi_{T}}[J^{\pi_{det}}(\mu)] - J^{\pi_{det}^*}(\mu) &\leq \frac{8\gamma^k}{1 - \gamma} g_{max}  + \frac{1}{T}\bigg(\beta D_{\Phi}(\tilde\pi^*,\tilde\pi_0) \bigg)\\
&\leq \frac{8\gamma^k}{1 - \gamma} g_{max}  + \frac{1}{T}\bigg(\frac{\beta}{2}\|\tilde\pi^* - \tilde \pi_0\|_2^2\bigg)
\end{align}
as desired.
\end{proof}

\subsubsection{Part 2.2) Mirror Descent}
\label{part2_mirror}
\begin{lemma}
Mirror descent with the same conditions as \cref{mirror_bound} converges and satisfies the following bound:

    $$\frac{1}{T}\sum_{t = 0}^{T - 1}-\nabla_{\tilde\pi^* - \tilde\pi_t} J^{\tilde\pi_t,k}(\mu)\leq \frac{1}{T}\bigg(\frac{\beta}{\lambda}D_{\Phi}(\tilde\pi^*,\tilde\pi_0)   \bigg) - \frac{1}{T}\sum_{t = 0}^{T - 1}(J^{\tilde\pi_{t + 1},k}(\mu) - J^{\tilde\pi_t,k}(\mu)) $$

    where $\pi^* = \delta(\pi^*_{det})\in \Pi_{res}$ is the Dirac delta of the optimal deterministic policy.
    \label[lemma]{proj_gd_lemma}

\end{lemma}

\begin{proof}

Mirror descent will converge due to the descent lemma shown in \cref{descent_lemma_mirror}.

To prove the bound, we will evoke the smoothness of the objective $J^{\tilde\pi',k}(\mu) \leq J^{\tilde\pi,k}(\mu) + \langle \nabla J^{\tilde\pi,k}(\mu), \tilde\pi' - \tilde\pi \rangle + \frac{\beta}{2} \|\tilde\pi' - \tilde\pi\|^2$. Next by the definition of the Bregman Divergence, $\mu$-strong convexity of $\Phi$ implies $D_{\Phi}(\tilde\pi',\tilde\pi) \geq \frac{\lambda}{2} \| \tilde\pi' - \tilde\pi \|^2$. Lastly, we will use a consequence of the three-point equality $\langle \nabla J^{\tilde\pi_t,k}(\mu), \tilde\pi_{t + 1} - \tilde\pi^*\rangle \leq\frac{1}{\eta}(D_{\Phi}(\tilde\pi^*,\tilde\pi_t) - D_{\Phi}(\tilde\pi^*, \tilde\pi_{t + 1}) - D_{\Phi}(\tilde\pi_{t + 1}, \tilde\pi_t))$ proved in \cref{three_point}.

We will create a bound on $-\nabla_{\tilde \pi^* - \tilde \pi_t}J^{\tilde\pi_t,k}$ as follows:
\begin{align}
-\nabla_{\tilde \pi^* - \tilde \pi_t}J^{\tilde\pi_t,k} &= \langle \nabla J^{\tilde\pi_t,k}(\mu), \tilde\pi_{t} - \tilde\pi^*\rangle\\
&= \langle \nabla J^{\tilde\pi_t,k}(\mu), \tilde\pi_{t} - \tilde\pi^*\rangle + \langle \nabla J^{\tilde\pi_t,k}(\mu), \tilde\pi_{t + 1} - \tilde\pi_t\rangle - \langle \nabla J^{\tilde\pi_t,k}(\mu), \tilde\pi_{t + 1} - \tilde\pi_t\rangle\\
&= \langle \nabla J^{\tilde\pi_t,k}(\mu), \tilde\pi_{t + 1} - \tilde\pi^*\rangle  - \langle \nabla J^{\tilde\pi_t,k}(\mu), \tilde\pi_{t + 1} - \tilde\pi_t\rangle\\
&\leq \langle \nabla J^{\tilde\pi_t,k}(\mu), \tilde\pi_{t + 1} - \tilde\pi^*\rangle + \frac{\beta}{2}\|\tilde \pi_{t + 1} - \tilde \pi_t\|^2 - (J^{\tilde\pi_{t + 1},k}(\mu) - J^{\tilde\pi_t,k}(\mu))\label{mirror_1}\\
&\leq \langle \nabla J^{\tilde\pi_t,k}(\mu), \tilde\pi_{t + 1} - \tilde\pi^*\rangle + \frac{\beta}{\lambda}D_{\Phi}(\tilde \pi_{t + 1}, \tilde \pi_t) - (J^{\tilde\pi_{t + 1},k}(\mu) - J^{\tilde\pi_t,k}(\mu))\label{mirror_2}\\
&\leq\frac{1}{\eta}(D_{\Phi}(\tilde\pi^*,\tilde\pi_t) - D_{\Phi}(\tilde\pi^*, \tilde\pi_{t + 1}) - D_{\Phi}(\tilde\pi_{t + 1}, \tilde\pi_t)) + \frac{\beta}{\lambda}D_\Phi(\tilde\pi_{t + 1}, \tilde\pi_t) - (J^{\tilde\pi_{t + 1},k}(\mu) - J^{\tilde\pi_t,k}(\mu))\label{mirror_3}\\
&=\frac{1}{\eta}(D_{\Phi}(\tilde\pi^*,\tilde\pi_t) - D_{\Phi}(\tilde\pi^*, \tilde\pi_{t + 1})) +  (\frac{\beta} {\lambda} - \frac{1}{\eta})D_\Phi(\tilde\pi_{t + 1}, \tilde\pi_t) - (J^{\tilde\pi_{t + 1},k}(\mu) - J^{\tilde\pi_t,k}(\mu))\label{mirror_4}\\
&\leq \frac{1}{\eta}(D_{\Phi}(\tilde\pi^*,\tilde\pi_t) - D_{\Phi}(\tilde\pi^*, \tilde\pi_{t + 1})) - (J^{\tilde\pi_{t + 1},k}(\mu) - J^{\tilde\pi_t,k}(\mu))
\end{align}
\cref{mirror_1}, \cref{mirror_2}, \cref{mirror_3} are the applications of the smoothness condition, lower bound on Bregman divergence, and \cref{three_point} respectively. In \cref{mirror_4}, our chosen learning rate $\eta = \frac{\lambda}{\beta}$ is used.

Averaging both sides across the iterations yields:
\begin{align}
\frac{1}{T}\sum_{t = 0}^{T - 1}-\nabla_{\tilde\pi^* - \tilde\pi_t} J^{\tilde\pi_t,k}(\mu)
&\leq \frac{1}{T\eta}\sum_{t = 0}^{T - 1}(D_{\Phi}(\tilde\pi^*,\tilde\pi_t) - D_{\Phi}(\tilde\pi^*, \tilde\pi_{t + 1})) - \frac{1}{T}\sum_{t = 0}^{T - 1}(J^{\tilde\pi_{t + 1},k}(\mu) - J^{\tilde\pi_t,k}(\mu))  \\
&= \frac{1}{T\eta}(D_{\Phi}(\tilde\pi^*,\tilde\pi_0) - D_{\Phi}(\tilde\pi^*, \tilde\pi_{T})) - \frac{1}{T}\sum_{t = 0}^{T - 1}(J^{\tilde\pi_{t + 1},k}(\mu) - J^{\tilde\pi_t,k}(\mu))    \\
&\leq \frac{1}{T}\bigg(\frac{\beta}{\lambda}D_{\Phi}(\tilde\pi^*,\tilde\pi_0)\bigg) - \frac{1}{T}\sum_{t = 0}^{T - 1}(J^{\tilde\pi_{t + 1},k}(\mu) - J^{\tilde\pi_t,k}(\mu))
\end{align}
where the learning rate $\eta = \frac{\lambda}{\beta}$ is use once again.

\end{proof}

This result is then combined with \cref{closed_form} to prove \cref{mirror_bound}.

\textbf{Proof for \cref{mirror_bound}}:
\begin{proof}
\cref{closed_form} is evoked with $\tilde \pi_t$ and $\tilde\pi^*$ which gives
$-\nabla_{\tilde\pi^* - \tilde\pi_{t}} J^{\tilde\pi_{t},k}(\mu) \geq \frac{1}{1 - \gamma^k}   (J^{\tilde\pi_{t},k}(\mu) - J^{\pi^*_{det}}(\mu)) - \frac{6\gamma^k}{(1 - \gamma^k)(1 - \gamma)} g_{max}$

Substituting into  \cref{proj_gd_lemma},
\begin{align}
\frac{1}{T}\sum_{t = 0}^{T - 1}J^{\tilde\pi_t,k}(\mu) -  J^{\pi^*_{det}}(\mu)
&\leq\frac{6\gamma^k}{1 - \gamma} g_{max} + \frac{1}{T}(1 - \gamma^k)\frac{\beta}{\lambda}D_{\Phi}(\tilde\pi^*,\tilde\pi_0)- \frac{1 - \gamma^k}{T}\sum_{t = 0}^{T - 1}(J^{\tilde\pi_{t + 1},k}(\mu) - J^{\tilde\pi_t,k}(\mu))\\
& \leq \frac{6\gamma^k}{1 - \gamma} g_{max}  + \frac{1}{T}\frac{\beta}{\lambda}D_{\Phi}(\tilde\pi^*,\tilde\pi_0)- \frac{1}{T}\sum_{t = 0}^{T - 1}(J^{\tilde\pi_{t + 1},k}(\mu) - J^{\tilde\pi_t,k}(\mu))\label{mirror_20}
\end{align}
Here \cref{mirror_20} uses that $\sum_{t = 0}^{T - 1}(J^{\tilde\pi_{t + 1},k}(\mu) - J^{\tilde\pi_t,k}(\mu))$ is negative by the descent lemma in \cref{descent_lemma_mirror}.

Combining the sums together with the descent lemma shown in \cref{descent_lemma_mirror} gives the last iterate performance bound
\begin{align*}
 J^{\tilde\pi_{T},k}(\mu) - J^{\pi^*_{det}}(\mu) &\leq \frac{1}{T}\sum_{t = 0}^{T - 1}J^{\tilde\pi_{t + 1},k}(\mu) -  J^{\pi^*_{det}}(\mu)\\
& \leq \frac{6\gamma^k}{1 - \gamma} g_{max}  + \frac{1}{T}\bigg(\frac{\beta}{\lambda}D_{\Phi}(\tilde\pi^*,\tilde\pi_0)\bigg)
\end{align*}

Following similar steps to \cref{part1_proof}, results in
\begin{align}
    &\mathbb E_{\pi_{det}\sim\tilde \pi_{T}}[J^{\pi_{det}}(\mu)] - J^{\pi_{det}^*}(\mu)\\
    &\leq J^{\tilde\pi_T,k}(\mu) -  J^{\pi^*_{det}}(\mu)
    + \frac{2\gamma^k}{1 - \gamma} g_{max}\\
& \leq \frac{8\gamma^k}{1 - \gamma} g_{max}  + \frac{1}{T}\bigg(\frac{\beta}{\lambda}D_{\Phi}(\tilde\pi^*,\tilde\pi_0) \bigg).
\end{align}

\end{proof}

\subsection{Intermediate Lemmas}
\begin{lemma}[$k$-step Performance Difference Lemma]
    $$J^{\tilde\pi_{1},k}(\mu) = J^{\tilde\pi_{2},k}(\mu) - \frac{1}{1 - \gamma^k}\mathbb E_{s\sim d_{\mu}^{\tilde\pi_{2}, k}}[A^{\tilde\pi_{1},k}(s, \tilde\pi_{2})]$$
    where $A^{\tilde \pi_1,k}(s, \tilde \pi_2) = Q^{\tilde\pi_1,k}(s, \tilde \pi_2) - J^{\tilde\pi_1,k}(s)$ is the $k$-step advantage function.
\end{lemma}
\begin{proof}
\label[lemma]{pd_lemma}
Following similar steps to the traditional performance difference lemma proof,
\begin{align}
    &J^{\tilde\pi_{1},k}(s_0) = Q^{\tilde\pi_{1},k}(s_0, \tilde\pi_{2}) - Q^{\tilde\pi_{1},k}(s_0, \tilde\pi_{2}) + J^{\tilde\pi_{1},k}(s_0)\\
    &= Q^{\tilde\pi_{1},k}(s_0, \tilde\pi_{2}) - A^{\tilde\pi_{1},k}(s_0, \tilde\pi_{2})\\
    &= \mathbb E_{\tau^k \sim\tilde\pi_{2}\lvert_{s_0}}[\sum_{t = 0}^{k - 1} \gamma^t g(s_t,a_t) + \gamma^k J^{\tilde\pi_{1},k}(s_{k})] - A^{\tilde\pi_{1},k}(s, \tilde\pi_{2})\\
    &= \ldots\\
    &= J^{\tilde\pi_{2},k}(s_0) - \sum_{m = 0}^\infty \gamma^{mk} \sum_{s \in \mathcal S}P(s_{mk} = s\lvert s_0, \tilde\pi_{2},k)A^{\tilde\pi_{1},k}(s_{mk}, \tilde\pi_{2})\label{pd_1}\\
    &= J^{\tilde\pi_{2},k}(s_0) - \frac{1}{1 - \gamma^k}\mathbb E_{s\sim d_{s_0}^{\tilde\pi_{2}, k}}[A^{\tilde\pi_{1},k}(s, \tilde\pi_{2})]
\end{align}
In \cref{pd_1}, we recursively substitute our expression for $J^{\tilde \pi_1,k}(s_0)$.

Taking expectations on both sides,
$J^{\tilde\pi_{1},k}(\mu) = J^{\tilde\pi_{2},k}(\mu) - \frac{1}{1 - \gamma^k}\mathbb E_{s\sim d_{\mu}^{\tilde\pi_{2}, k}}[A^{\tilde\pi_{1},k}(s, \tilde\pi_{2})]$
\end{proof}

\begin{lemma}
\label[lemma]{dist_close}
    $\sum_{s\in \mathcal S}\lvert \mu(s) - d_\mu^{\tilde\pi,k}(s) \rvert \leq 2\gamma^k$
\end{lemma}
\begin{proof}
The following bounds the difference between $\mu$ and $d_\mu^{\tilde\pi,k}$:
\begin{align}
    \sum_{s\in \mathcal S} \lvert \mu(s) - d_\mu^{\tilde\pi,k}(s) \rvert &= \sum_{s\in \mathcal S} \bigg\lvert \mu(s) - (1 - \gamma^k) \sum_{m = 0}^\infty \gamma^{mk} P(s_k = s \lvert \mu, \tilde\pi,k)\bigg\rvert\\
    & = \sum_{s\in \mathcal S} \bigg\lvert \mu(s) - (1 - \gamma^k) \mu(s) - (1 - \gamma^k)\sum_{m = 1}^\infty \gamma^{mk} P(s_k = s \lvert \mu, \tilde\pi,k)\bigg\rvert\\
    & = \sum_{s\in \mathcal S} \bigg\lvert \gamma^k \mu(s) - \gamma^k(1 - \gamma^k)\sum_{m = 1}^\infty \gamma^{(m - 1)k} P(s_k = s \lvert \mu, \tilde\pi,k)\bigg\rvert\\
    & = (2\gamma^k)\frac{1}{2}\sum_{s\in \mathcal S} \bigg\lvert  \mu(s) - (1 - \gamma^k)\sum_{m' = 0}^\infty \gamma^{m'k} P(s_k = s \lvert \mu, \tilde\pi,k)\bigg\rvert\\
    &\leq 2\gamma^k.
\end{align}
The last step uses that the total variation distance of any two distributions is bounded by $1$.
\end{proof}

\begin{lemma}
Mirror descent under the same assumptions of \cref{mirror_bound} will satisfy
    $$\langle \nabla J^{\tilde\pi_t,k}(\mu), \tilde\pi_{t + 1} - \tilde\pi\rangle \leq \frac{1}{\eta}(D_{\Phi}(\tilde\pi,\tilde\pi_t) - D_{\Phi}(\tilde\pi, \tilde\pi_{t + 1}) - D_{\Phi}(\tilde\pi_{t + 1}, \tilde\pi_t))$$
    \label[lemma]{three_point}
    for any $\tilde \pi\in \tilde \Pi_{res}$.
\end{lemma}
\begin{proof}
Let $F_t(\tilde\pi) = \eta \langle \nabla J^{\tilde\pi_t,k}(\mu), \tilde\pi\rangle + D_{\Phi}(\tilde\pi, \tilde\pi_t)$. Note that $\tilde\pi_{t + 1}$ is a minimizer of this function by definition.

Further, $F_t$ is convex in $\tilde \pi$ and therefore $\tilde\pi_{t + 1}$ is globally optimal for $F_t$. Therefore,
$\langle \nabla F_t(\tilde\pi_{t + 1}), \tilde\pi - \tilde\pi_{t + 1}\rangle \geq 0$ or equivalently $\langle \nabla F_t(\tilde\pi_{t + 1}), \tilde\pi_{t + 1} - \tilde\pi\rangle \leq 0$.

$\nabla F_t(\tilde \pi)$ can be simplified using the definition of the Bregman Divergence $D_\Phi(x,y) = \Phi(x) - \Phi(y) - \nabla \Phi(y)^T(x - y)$ as follows: $\nabla F_t(\tilde\pi) = \eta \nabla J^{\tilde\pi_t,k}(\mu) +  \nabla D_{\Phi}(\tilde\pi, \tilde\pi_t) =  \eta \nabla J^{\tilde\pi_t,k}(\mu) + \nabla \Phi(\tilde\pi) - \nabla \Phi(\tilde\pi_t)$.

Plugging this expression into $\langle \nabla F_t(\tilde\pi_{t + 1}), \tilde\pi_{t + 1} - \tilde\pi\rangle \leq 0$, we get

$\eta \langle \nabla J^{\tilde\pi_t,k}(\mu), \tilde\pi_{t + 1} - \tilde\pi\rangle \leq \langle \nabla \Phi(\tilde\pi_t) - \nabla\Phi(\tilde\pi_{t + 1}), \tilde\pi_{t + 1} - \tilde\pi\rangle = D_{\Phi}(\tilde\pi,\tilde\pi_t) - D_{\Phi}(\tilde\pi, \tilde\pi_{t + 1}) - D_{\Phi}(\tilde\pi_{t + 1}, \tilde\pi_t)$
using the three point property of Bregman Divergences $(\nabla \Phi(x) - \nabla \Phi(y))^T(x - z) = D_\Phi(x,y) + D_\Phi(z,x) - D_\Phi(z,y)$
as desired.
\end{proof}
\begin{lemma}
Mirror descent under the same assumptions of \cref{mirror_bound} will satisfy
$$J^{\tilde\pi_{t + 1},k}(\mu) - J^{\tilde\pi_t,k}(\mu) \leq -\frac{\lambda}{2\eta}\|\tilde\pi_{t + 1} - \tilde\pi_t\|^2 $$
    \label[lemma]{descent_lemma_mirror}
\end{lemma}
\begin{proof}

By the $\beta$-smoothness of $J^{\tilde \pi,k}(\mu)$, we have
\begin{align}
    &J^{\tilde\pi_{t + 1},k}(\mu) \leq J^{\tilde\pi_t,k}(\mu) + \langle \nabla J^{\tilde\pi_t,k}(\mu), \tilde\pi_{t + 1} - \tilde\pi_t\rangle + \frac{\beta}{2}\|\tilde\pi_{t + 1} - \tilde\pi_t\|^2\\
&\leq J^{\tilde\pi_t,k}(\mu) - \frac{1}{\eta}(D_\Phi(\tilde\pi_{t + 1}, \tilde\pi_t) + D_\Phi(\tilde\pi_{t}, \tilde\pi_{t + 1}) )  + \frac{\beta}{2}\|\tilde\pi_{t + 1} - \tilde\pi_t\|^2\label{mirror_1_0}\\
&\leq J^{\tilde\pi_t,k}(\mu) - \frac{\lambda}{\eta}\|\tilde\pi_{t + 1} - \tilde\pi_t\|^2 + \frac{\beta}{2}\|\tilde\pi_{t + 1} - \tilde\pi_t\|^2\label{mirror_1_1}\\
&= J^{\tilde\pi_t,k}(\mu) - (\frac{\lambda}{\eta} - \frac{\beta}{2})\|\tilde\pi_{t + 1} - \tilde\pi_t\|^2 \\
&= J^{\tilde\pi_t,k}(\mu) - \frac{\lambda}{2\eta}\|\tilde\pi_{t + 1} - \tilde\pi_t\|^2  \label{mirror_1_2}.
\end{align}
In \cref{mirror_1_0}, we evoke \cref{three_point} with $\tilde \pi_t$ as $\tilde \pi$.
In \cref{mirror_1_1}, we evoke the $\lambda$-strong convexity of $\Phi$, the Bregman Divergence satisfies $D_{\Phi}(\tilde\pi',\tilde\pi) \geq \frac{\lambda}{2} \| \tilde\pi' - \tilde\pi \|^2$. In \cref{mirror_1_2}, we use $\eta = \frac{\lambda}{\beta}$.
\end{proof}

\section{Potential Pitfall - Alternative $k$-step Definition:}
The following is a tempting but unsuccessful alternative $k$-step construction that at first glance may appear promising. Suppose we return to the original policy definition $\pi: \mathcal S \rightarrow \Delta(\mathcal A)$ and simply maintain $k$ independent policies $\{\pi^j: \pi^j\in \Pi_{res}, 0 \leq j < k\}$ that would cycled through and executed one after the other independently.

This seems as though it may allow extra ``directions of movement'' that would allow a descent on these chain of policies to escape poor suboptimal local minima. Unfortunately, it can be verified by simulation that the suboptimal local minima in the two state example from \cref{running_example} is not removed with this scheme even for large k.

Intuitively, the reason is because the argument from \cref{origin} still causes myopicness in this alternative scheme. The two state example under this scheme will have policies of the form $\pi_{\theta_1}, \pi_{\theta_1},\ldots, \pi_{\theta_k}$ that can have different $\theta$-parameters at each step. But the policy with $\theta_i = 0$ for $0 \leq i < k$ will remain a critical point because when perturbed by $d\theta_i$ for parameter $\theta_i$, the probability of sampling trajectories taking action R more than one time will include products of these $d\theta_i$. Similar to \cref{origin}, these will be treated as higher order terms.

\newpage

\end{document}